\documentclass[11pt,letterpaper]{article}

\usepackage[margin=1in]{geometry}        
\usepackage{setspace}                     
\usepackage[T1]{fontenc}                  
\usepackage[utf8]{inputenc}               
\usepackage{lmodern} 
\usepackage{graphicx}
\usepackage{natbib}
\usepackage{authblk}
\usepackage{enumerate}
\usepackage{enumitem}
\usepackage{amsmath,amssymb,amsthm,mathtools,bbm}

\numberwithin{equation}{section}

\theoremstyle{plain}
\newtheorem{theorem}{Theorem}[section]
\newtheorem{lemma}[theorem]{Lemma}

\newtheorem{corollary}[theorem]{Corollary}

\theoremstyle{definition}
\newtheorem{definition}{Definition}

\usepackage[ruled, linesnumbered]{algorithm2e}
\usepackage{xcolor}

\usepackage[
  colorlinks,
  linkcolor=blue,
  citecolor=blue,
  urlcolor=blue,
  breaklinks
]{hyperref}

\newcommand{\N}{\mathbb{N}} 
\def \iidsim {{\overset{iid}{\sim}\,}}

\def\Dcal {{\mathcal{D}}}
\def\Fcal {{\mathcal{F}}}
\def\E {{\mathbb{E}}}
\def\P {{\mathbb{P}}}

\def \mle {{\,\mathrm{MLE}}}
\newcommand{\One}[1]{{\mathbbm{1}}\left\{{#1}\right\}}

\newcommand{\bQ}{\mathbb{Q}}

\newcommand{\hel}{\textnormal{d}_{\textnormal{H}}}

\renewcommand{\d}{\mathrm{d}}

\DeclareMathOperator*{\argmax}{argmax}

\newcommand{\papertitle}{Online monotone density estimation and log-optimal calibration}
\newcommand{\paperauthorA}{Rohan Hore}
\newcommand{\paperauthorB}{Ruodu Wang}
\newcommand{\paperauthorC}{Aaditya Ramdas}
\newcommand{\affilOne}{Department of Statistics and Data Science, Carnegie Mellon University, USA}
\newcommand{\affilTwo}{Department of Statistics and Actuarial Science, University of Waterloo, Canada}

\newcommand{\corrEmail}{rhore@andrew.cmu.edu}

\title{\papertitle}

\author[1]{\paperauthorA\thanks{Corresponding author: \corrEmail}}
\author[2]{\paperauthorB}
\author[1]{\paperauthorC}
\affil[1]{\affilOne}
\affil[2]{\affilTwo}

\date{\today}

\newcommand{\paperabstract}[1]{%
  \begin{abstract}
    #1
  \end{abstract}
}

\usepackage[nameinlink,noabbrev]{cleveref}
\begin{document}

\maketitle

\paperabstract{%
We study the problem of online monotone density estimation, where density estimators must be constructed in a predictable manner from sequentially observed data. We propose two online estimators: an online analogue of the classical Grenander estimator, and an expert aggregation estimator inspired by exponential weighting methods from the online learning literature. In the well-specified stochastic setting, where the underlying density is monotone, we show that the expected cumulative log-likelihood gap between the online estimators and the true density admits an $O(n^{1/3})$ bound. We further establish a $\sqrt{n\log{n}}$ pathwise regret bound for the expert aggregation estimator relative to the best offline monotone estimator chosen in hindsight, under minimal regularity assumptions on the observed sequence.
As an application of independent interest, we show that the problem of constructing log-optimal p-to-e calibrators for sequential hypothesis testing can be formulated as an online monotone density estimation problem. We adapt the proposed estimators to build empirically adaptive p-to-e calibrators and establish their optimality. Numerical experiments illustrate the theoretical results.
}

\section{Introduction}
In this work, we study the problem of online monotone density estimation. Let $\Dcal$ be the class of non-increasing densities supported on $[0,1]$, defined by
\begin{equation}\label{eqn:Dcal}
    \Dcal:=\biggl\{f:[0,1]\to[0,\infty):\int f(u)\,\d u=1,\; f(x)\le f(y)\ \textnormal{for all}\ x\ge y\biggr\}.
\end{equation}
Suppose $X_1,X_2,\ldots\iidsim\bQ$ are observations drawn sequentially, where $\bQ$ admits a density $q=\frac{\d\bQ}{\d\lambda}\in \Dcal$. Throughout this paper, $\lambda$ is the Lebesgue measure on $[0,1]$. Our aim is to construct an \emph{algorithm} $\hat f$, by which we mean a predictable sequence of density estimators $\hat f=(\hat f_t)_{t\ge1}$. Formally, at each time $t\ge2$, the algorithm outputs a density $\hat f_t$, a measurable function of the past observations $(X_1,\ldots,X_{t-1})$, that is, $\hat f_t$ is $\mathcal F(X_1,\ldots,X_{t-1})$-measurable, where $\mathcal F(X_1,\ldots,X_{t-1})$ denotes the $\sigma$-algebra generated by $(X_1,\ldots,X_{t-1})$.

Restricting attention to densities supported on a compact interval is standard in monotone density estimation and is often essential for meaningful risk bounds; see, for example, \citep[Section~1]{gao2009rate}. While in this paper, we assume densities are supported on $[0,1]$, any compact interval can be studied by suitable translation and rescaling.

\subsection{Risk and regret for online monotone density estimation}
We measure the performance of an online algorithm $\hat{f}$ through its sequential log-likelihood. In particular, for a predictable density $f\in\Dcal$, we define $\ell_t(f):=-\log f(X_t)$, the negative log-likelihood of the observation $X_t$, as the instantaneous loss at time $t$. Accordingly, the cumulative loss incurred by an online algorithm $\hat f$ up to time $n$ is then given by
\[
\mathcal L(\hat f,n)=\sum_{t=1}^n \ell_t(\hat f_t)=-\sum_{t=1}^n \log \hat f_t(X_t).
\]
For any fixed
%static estimator
$f$, we will write $\mathcal L(f,n)$ to denote $-\sum_{t=1}^n \log f(X_t)$. We study this loss under both stochastic and adversarial regimes, comparing the online algorithm with the population-level and hindsight-optimal benchmarks.

\paragraph{Stochastic well-specified setting.}
When $X_1,X_2,\ldots$ are iid observations from $\bQ$, with density $q\in\Dcal$, the true density $q$ is Bayes-optimal under log-loss. This motivates the following notion of \emph{excess KL-risk},
\begin{equation}\label{eq:risk}
    \mathrm{Risk}_\bQ(\hat f;n)=\E_\bQ\bigl[\mathcal L(\hat f,n)\bigr]
-\min_{f\in\Dcal}\E_\bQ\bigl[\mathcal L(f,n)\bigr]
=\E_\bQ\bigl[\mathcal L(\hat f,n)\bigr]-\E_\bQ\bigl[\mathcal L(q,n)\bigr].
\end{equation}
For any online algorithm $\hat f$, since $\hat f_s$ is $\mathcal F_{s-1}$-measurable, by the tower property, we can express excess KL-risk equivalently as $\sum_{s=1}^n\E\bigl[\mathrm{KL}(q\|\hat f_s)\bigr]$.

\paragraph{Adversarial setting.}
We also compare an online algorithm against the best monotone density that could be selected offline with full access to the data up to a fixed time $n\ge 1$, in order to assess how well the online procedures track a hindsight-optimal benchmark. For
$0\le a\le b\le\infty$, let
\[
\Dcal_{a,b}:=\bigl\{f\in\Dcal:a\le f(u)\le b \ \textnormal{for all } u\in[0,1]\bigr\}
\]
denote the subclass of $\Dcal$ consisting of densities that are uniformly bounded below and above by $a$ and $b$, respectively. A natural offline benchmark is given by the solution of the maximum likelihood problem constrained to $\Dcal_{a,b}$,
\begin{equation}\label{eq:constrained_offline_grenander}
\tilde f^{\mle}_{n,a,b}\in\argmax_{f\in\Dcal_{a,b}}\sum_{i=1}^n \log f(X_i).
\end{equation}
The unconstrained case $a=0$ and $b=\infty$ was originally introduced by \citet{grenander1956theory} and is known as the Grenander density estimator in the literature; therefore, we refer to \eqref{eq:constrained_offline_grenander} as the \emph{constrained Grenander estimator}. 
Comparing an online algorithm $\hat f$ to this benchmark leads to:
\begin{equation}\label{eq:regret}
    \mathrm{Regret}(\hat f;n,a,b):=\mathcal L(\hat f,n)-\min_{f\in\Dcal_{a,b}}\mathcal L(f,n),
\end{equation}
which provides a pathwise performance measure without imposing any distributional assumptions.

For theoretical analysis, we will often restrict attention to the bounded class $\Dcal_{a,b}$. In practice, however, one may take $a=0$ and $b=\infty$, recovering the classical Grenander estimator. Ideally, a well-designed online procedure should adapt automatically to the effective range of the density, without requiring prior knowledge of such bounds. Empirically, we indeed observe this behavior for all online procedures developed in this work.

\subsection{The Grenander estimator and its properties}

\citet{grenander1956theory} characterized the Grenander estimator as the left-continuous density corresponding to the least concave majorant of the empirical distribution function. Consequently, the Grenander estimator is simply a monotone histogram that is strictly positive on $[0,\max_{1\le i\le n} X_i]$. Throughout this work, we extend the estimator beyond $\max_{1\le i\le n} X_i$ by a right-continuous constant equal to the final histogram height and suitably renormalize it, thereby ensuring positivity and finite log-loss evaluations. Since the foundational work of \citet{grenander1956theory}, the theoretical properties of the Grenander estimator have been studied extensively; see \citet{durot2012limit} for a comprehensive historical overview.

Early work by \citet{rao1969estimation} and \citet{carolan1999asymptotic} characterized the asymptotic distribution of $\tilde f^{\mathrm{MLE}}_n(x)$ in regions where the true density is locally strictly decreasing and locally flat, respectively. Later studies established sharp global convergence rates under various loss functions. In particular, convergence at rate $n^{-1/3}$ in Hellinger distance was proved in \citet{van1993hellinger,gao2009rate}, while asymptotic normality of the $L_1$ error was analyzed in \citet{groeneboom1999asymptotic}. Additional asymptotic and non-asymptotic properties of the Grenander estimator have been developed in \citet{balabdaoui2010estimation,birge1989grenader,kulikov2005asymptotic,groeneboom1993isotonic}.

\subsection{Our contributions}
% Our main contributions are summarized below.

\paragraph{Online monotone density estimation.}
Despite the extensive offline literature, monotone density estimation has not been systematically studied in the online setting. This work fills this gap.
\begin{itemize}
    \item We propose two online monotone density estimators: an online analogue of the classical Grenander estimator, and an expert aggregation estimator based on exponential weighting methods from online learning.
    \item In the well-specified stochastic setting, where the true density $q$ is monotone, we establish statistical guarantees for both estimators, showing that their excess KL-risk is of order $n^{1/3}$.
    \item In the adversarial setting, we derive a pathwise regret bound of order $\sqrt{n\log n}$ for the expert aggregation estimator, under mild regularity conditions on the observed sequence.
    \item The pathwise regret bound suggests improved finite-sample adaptivity of the expert aggregation estimator relative to the online Grenander estimator in a mis-specified stochastic setting, a phenomenon that we confirm empirically through numerical simulations.
\end{itemize}

\paragraph{Application to sequential testing.}
We develop a non-trivial application of online monotone density estimation to the
construction of sequential hypothesis tests.
\begin{itemize}
    \item While most classical tests are formulated in terms of p-values, e-values provide a flexible framework for sequential inference under optional stopping. We observe that the problem of constructing optimal p-to-e calibrators reduces to online monotone density estimation.
    \item We adapt the proposed online density estimators to develop empirically adaptive p-to-e calibrators, yielding valid sequential tests. If under the alternative, the p-values are iid with a monotone density, the resulting calibrators are asymptotically log-optimal.
\end{itemize}

\paragraph{Organization of the paper.}
% The remainder of the paper is organized as follows.
Section~\ref{sec:online_monotone_density_estimation} introduces the proposed online monotone density estimators and establishes bounds on their excess KL-risk and regret. Section~\ref{sec:application_to_sequential_testing} presents an application to sequential testing, yielding valid procedures with provable optimality guarantees. Finally, Section~\ref{sec:simulations} illustrates the theoretical results through a series of numerical experiments.

\section{Online monotone density estimation}\label{sec:online_monotone_density_estimation}
In this section, we introduce two online algorithms and then establish their theoretical properties.

\paragraph{Online Grenander algorithm.}
Motivated by the classical Grenander estimator $\tilde f_n^{\mle}$, a natural approach is to adapt the maximum likelihood formulation to the online setting. Specifically, for any $0\le a\le b\le\infty$, the online Grenander (OG) algorithm $\hat f^{\,\mathrm{OG}}_{a,b}=(\hat f^{\,\mathrm{OG}}_{t,a,b})_{t\ge1}$ produces, at each time $t\ge2$, the estimator
\begin{equation}\label{eq:online_grenander}
\hat f^{\,\mathrm{OG}}_{t,a,b}\in\argmax_{f\in\mathcal D_{a,b}}\sum_{i=1}^{t-1}\log f(X_i),\qquad \hat f^{\,\mathrm{OG}}_{1,a,b}\equiv 1.
\end{equation}
This algorithm is the direct online analogue of the constrained offline MLE $\tilde f_{n,a,b}^{\mle}$. While statistically natural, it recomputes a full monotone MLE at each time step.

\paragraph{Expert aggregation algorithm.}
To address the computational inefficiency of OG, we consider an aggregation-based alternative, well studied in the online learning literature \citep{cesa2006prediction,hao2018online}. Fix any $0\le a\le b\le\infty$ and a finite collection of monotone density experts
\[
\mathcal E_m=\{f_1,\ldots,f_m\}\subseteq\mathcal D_{a,b}.
\]
The expert aggregation (EA) algorithm $\hat f^{\,\mathrm{EA}}_{a,b}=(\hat f^{\,\mathrm{EA}}_{t,a,b})_{t\ge1}$ outputs, at each time $t\ge1$, the estimator
\begin{equation}\label{eq:expert_aggregation}
\hat f^{\,\mathrm{EA}}_{t,a,b}(u)=\sum_{f\in\mathcal E_m} w_t(f)\,f(u),
\qquad
w_t(f)=\frac{\prod_{i=1}^{t-1} f(X_i)}{\sum_{f'\in\mathcal E_m}\prod_{i=1}^{t-1} f'(X_i)},\qquad u\in[0,1].
\end{equation}
Thus, the EA algorithm produces a predictable sequence of estimators formed as convex combinations of experts, with weights favoring densities that assign higher likelihood to past observations.

\subsection{Risk bounds for online monotone density estimation}
We now bound the excess KL-risk, defined in~\eqref{eq:risk}, for both the OG algorithm \eqref{eq:online_grenander} and the EA algorithm \eqref{eq:expert_aggregation}, which serves as a measure of the statistical accuracy of the proposed online procedures in estimating the true monotone density $q$. Throughout this section, we assume that
\[
X_1,X_2,\ldots \stackrel{iid}{\sim} \bQ,
\qquad q:=\frac{\d\bQ}{\d\lambda}\in\Dcal.
\]
To place our results in context, we first recall some relevant theoretical results of the Grenander estimator. If $q$ is bounded above, then the Grenander estimator satisfies $\d_{\mathrm H}(q,\tilde f^{\mle}_n)=O_p(n^{-1/3})$ 
% \RW{$O_p$?} 
\citep{gao2009rate,van1993hellinger}, where $\d_{\mathrm H}$ denotes the Hellinger distance. When both $q$ and the estimator are additionally bounded away from zero and infinity, this result implies a corresponding $n^{-2/3}$ rate for the expected Kullback--Leibler divergence. 
% For completeness, we recall the definitions of these divergences:
% \begin{equation*}
%     \d_{\mathrm H}^2(p,r):=\frac{1}{2}\int_0^1\bigl(\sqrt{p(x)}-\sqrt{r(x)}\bigr)^2\,\d x,
% \qquad \mathrm{KL}(p\|r)
% :=\int p(x)\log\!\left(\frac{p(x)}{r(x)}\right)\,\d x,
% \end{equation*}
% where $p,r$ are densities supported on $[0,1]$.
We state the precise implication for the KL risk formally below, for later use in our analysis.

\begin{lemma}\label{lem:KL_rate_of_montone_mle}
Fix $n\in\mathbb N$. Suppose that there exist $0<a<b<\infty$ such that $q\in\mathcal{D}_{a,b}$. Then, the constrained Grenander estimator $\hat{f}^{\mathrm{MLE}}_{n,a,b}\equiv \hat{f}^{\mathrm{MLE}}_{n,a,b}(X_1,\ldots,X_n)$ satisfies 
% \RW{How do we interpret the randomness in $\hat{f}^{\mathrm{MLE}}_{n,a,b}$here? Should we clarify it? It is a function.}{RH: how about this? Is it good enough?} 
\[\E_{\bQ}\bigl[\mathrm{KL}\bigl(q\|\hat{f}^{\mathrm{MLE}}_{n,a,b}\bigr)
\bigr]\le\kappa\,n^{-2/3}\]
for some constant $\kappa>0$ depending only on $a$ and $b$.
\end{lemma}
We are now ready to state the bounds on excess KL-risk for OG and the EA estimators. 
\begin{theorem}[Risk bounds for OG and EA algorithms]\label{thm:excess_KL_risk_OG_and_EA}
Fix $n\in\mathbb N$ and $0<a<b<\infty$. For any $q\in\mathcal D_{a,b}$ and $X_1,\ldots,X_n\stackrel{iid}{\sim}\mathbb Q$, the OG algorithm satisfies
\[
\mathrm{Risk}_\bQ(\hat{f}^{\,\mathrm{OG}}_{a,b};n)\le\Gamma_{\mathrm{OG}}(a,b)\,n^{1/3},
\]
and there exists a collection of monotone density experts $\mathcal E_{m_n}$ with size $m_n<\infty$ such that for all $q\in\mathcal D_{a,b}$, the corresponding EA algorithm satisfies
\[
\mathrm{Risk}_\bQ(\hat{f}^{\,\mathrm{EA}}_{a,b};n)
\le\Gamma_{\mathrm{EA}}(a,b)\,n^{1/3},
\]
where the constants $\Gamma_{\mathrm{OG}}(a,b)$ and
$\Gamma_{\mathrm{EA}}(a,b)$ depend only on $a$ and $b$.
\end{theorem}
The rate in Theorem~\ref{thm:excess_KL_risk_OG_and_EA} is essentially optimal. Indeed, in the offline setting, monotone density estimation is known to attain minimax $L_1$ risk of order $n^{-1/3}$, corresponding to excess-KL risk of order $n^{1/3}$. Moreover, any substantial improvement in the online rate would yield an improvement in the offline problem through a standard online-to-offline conversion argument. Specifically, given online estimators $\hat f_1,\ldots,\hat f_n$, one may construct an offline estimator via averaging,
\[
\bar f_n:=\frac1n\sum_{t=1}^n \hat f_t,
\]
thereby transferring cumulative online risk guarantees to the offline setting.

Although both procedures achieve the same worst-case rate, their finite-sample behavior can differ substantially. In particular, our empirical results suggest that the EA procedure often attains smaller regret in practice, reflecting improved empirical adaptivity.

\subsection{Bounding regret for the EA algorithm }\label{sec:bounding_regret_for_EA}
In this subsection, we fix $n\in\N$, $0<a<b<\infty$ and derive a pathwise regret bound for the EA algorithm under mild spacing and boundary conditions on the observed sequence $(x_1,x_2,\ldots,x_n)$. The proof is based on a carefully chosen finite class of monotone histogram experts, that we define below.

We begin by introducing a binning of the unit interval. For $\beta>0$ and $n\ge1$, define the grid
\[
\mathcal B_n := \left\{0,\Delta_b,2\Delta_b,\ldots,\lfloor n^{(\beta+1)}\rfloor\Delta_b\right\},
\qquad
\Delta_b := n^{-(\beta+1)}.
\]
% \RW{Unclear how this is defined when $n^{\beta+1}$ is not an integer. Please clarify.}
Building upon this grid, we construct a finite class of monotone histogram densities. For integers $k\ge 1$, let $\mathcal{E}_{k,n,\beta}$ denote the collection of functions $g\in\Dcal_{a,b}$ of the form
\[
g(x)=\sum_{j=1}^r \theta_j\,\One{t_{j-1}\le x<t_j},
\]
where the parameters satisfy the following:
\begin{itemize}[itemsep=0.9pt, topsep=2.2pt]
\item[(i)] the number of bins $r$ satisfies $r\le k$;
\item[(ii)] the breakpoints $t_1,\ldots,t_{r-1}$ belong to the grid $\mathcal B_n$, and $t_0=0, t_r=1$; 
% \RW{Do we need to add $b_r=1$?}
\item[(iii)] the heights are non-increasing, $\theta_1\ge\cdots\ge\theta_{r-1}$, with $\log\theta_j\in\Lambda_{n,a,b}$ for $j=1,\ldots,r-1$, where $\Lambda_{n,a,b}:=\log(a)+\log(b/a)\,\mathcal B_n$;  
\item[(iv)] the final height $\theta_r$ is defined as $\theta_r:=\frac{1-\sum_{j=1}^{r-1}\theta_j(t_j-t_{j-1})}{1-t_{r-1}}$, and is required to satisfy that $\theta_r\in[a,b]$ and $\theta_{r-1}\ge\theta_r$.
\end{itemize} 
We show in the next section that, for fixed $(n,k,\beta)$, we derive a bound on the size of $\mathcal E_{k,n,\beta}$.
Classical discretizations of monotone density classes, as in \cite{gao2007entropy}, are typically designed to control approximation error in metrics such as $L_p$ or Hellinger distance. Such guarantees are insufficient for the present analysis, which requires control of cumulative log-loss and hence control on pointwise likelihood ratios along the realized sample path. Our construction therefore refines these classical discretization ideas to additionally provide suitable pointwise control while preserving monotonicity.

It is worth emphasizing that the EA algorithm itself does not require knowledge of $n$ for implementation. Indeed, the procedure simply performs sequential exponential weighting over a finite collection of experts. The dependence of the expert class on $n$ arises here solely as a technical device to facilitate a sharp theoretical analysis, enabling us to derive optimal  regret guarantee.
% \RW{This statement is too strong. It is exponential in $k$ (actually $n^k$) so it may not be feasible even if finite. We should not make such a statement unless we have a polynomial algorithm}{RH: valid point}

Before stating the regret bound, we formalize the regularity conditions imposed on the data sequence. For $\beta,\gamma>0$ and $n\ge 1$, we define the \emph{good} set
\begin{equation}\label{eq:good_set}
    \mathcal{S}_n(\beta,\gamma):=\Bigl\{(x_1,\ldots,x_n)\in[0,1]^n:
\min_{1\le i<j\le n}|x_i-x_j|\ge n^{-\beta},\ \max_{1\le i\le n}x_i\le 1-n^{-\gamma}\Bigr\}.
\end{equation}
Elements of $\mathcal S_n(\beta,\gamma)$ are well separated, with minimum spacing of order $n^{-\beta}$, and remain uniformly away from the boundary point $1$ by a margin of order $n^{-\gamma}$. These conditions ensure stability of the discretization with respect to the pathwise loss along the observed sample path from $\mathcal S_n(\beta,\gamma)$. In particular, the rounding of breakpoints and heights induced by the experts in $\mathcal E_{k,n,\beta}$ introduces only a controlled perturbation in likelihood evaluations, effectively yielding local $\ell_\infty$ control along the sample path. We are now ready to formally state the regret bound.
\begin{theorem}\label{thm:pathwise_regret_monotone_density}
Fix $\beta>0$, $\gamma\in (0,\beta-\tfrac12)$. For  $n\ge1$, and for all $(x_1,\ldots,x_n)\in\mathcal S_n(\beta,\gamma)$,
the EA algorithm $\hat f^{\mathrm{EA}}_{a,b}$ based on the set of experts $\mathcal E_{k,n,\beta}$ with  $k=\lfloor(n/\log{n})^{1/2}\rfloor$ satisfies
\[
\mathrm{Regret}(\hat f^{\,\mathrm{EA}}_{a,b};n,a,b)
\;\le\;
\Gamma(a,b,\beta)\,\sqrt{n\log n},
\]
where $\Gamma(a,b,\beta)$ is a constant depending only on $a$, $b$, and $\beta$. 
% \RW{not $k$? can I choose $k=1$?}{RH: k needs to be chosen adaptively..it is later chosen as $\sqrt{n}$ in the proof} \RW{Let's add: where $k$ depends on $n$ or specifically: $k=\sqrt {n/\log n}$. We cannot say ``the set of experts" because it is not fully specified ($k$ is not specified).}{Thanks for pointing out!} \RW{I think $k\asymp \sqrt {n/\log n} $ is OK for the statement, which is what we proved.} \RW{Sorry, I think it is better to use $=$ now, because I think we really want $\Gamma$ to not depend on this choice.}{RH: this looks good!}
\end{theorem}

Although the regularity conditions encoded in $\mathcal S_n(\beta,\gamma)$ may appear restrictive, they are mild in the stochastic setting. In particular, when $X_1,\ldots,X_n\stackrel{iid}{\sim}\bQ$ with $q\in\Dcal$, for any choices of $\beta>2$ and any $\gamma>1$, the sample path $(X_1,\ldots,X_n)$ belongs to $\mathcal S_n(\beta,\gamma)$ with high probability. We formalize these statements in Lemma~\ref{lem:min_distance_in_dimension_d}, given in Appendix~\ref{app:supplementary_results}.
As an immediate consequence, we get the following high-probability regret bound under the iid alternative model.
\begin{corollary} 
Suppose $X_1,\ldots,X_n\stackrel{iid}{\sim}\bQ$ with $q\in\Dcal_{a,b}$. Consider the EA algorithm $\hat{f}^{\,\mathrm{EA}}_{a,b}$ based on the set of experts $\mathcal E_{k,n,\beta}$ with  $k=\lfloor(n/\log{n})^{1/2}\rfloor$ for any $\beta>2$. Then, for any $\delta\in(0,1)$, there exists $N\equiv N(\delta,\beta,b)$ such that for all $n>N$,
\[
\P_{\bQ}\left(\mathrm{Regret}(\hat f^{\,\mathrm{EA}}_{a,b};n,a,b)\;\le\;C\,\sqrt{n\log n}\right)\ge 1-\delta,
\]
where $C$ is a constant, only depending on $a,b,\beta$.
\end{corollary}

\subsection{Proof sketch for the pathwise regret bound (Theorem~\ref{thm:pathwise_regret_monotone_density})}

Here, we outline the main ideas underlying the proof of the pathwise regret bound, deferring the proofs of the technical lemmas to Appendix~\ref{app:completing_proof_of_pathwise_regret}. Throughout this section, we write $V:=\log(b/a)$. We also write $\mathrm{Hist}_{k,\downarrow}$ to denote the class of monotone histograms with $k$ bins, that is
% \begin{multline*}
%     \mathrm{Hist}_{k,\downarrow}:=\biggl\{h(x)=\sum_{i=1}^k \theta_i\,\One{t_{i-1}\le x<t_i}:  ~~0=t_0<t_1<\cdots<t_k=1,\\
%     b\ge \theta_1\ge \ldots\ge \theta_k\ge a,\quad \sum_{j=1}^k \theta_j(t_j-t_{j-1})=1\biggr\}.
% \end{multline*}
% \RW{See whether the definition below is better (one line). Note that I did not order $t$s and $\theta$s. I guess we don't need it anyway.}{RH: yeah, happy with it!}
\begin{equation*}
    \mathrm{Hist}_{k,\downarrow}:= \biggl\{h: x\mapsto\sum_{i=1}^k \theta_i\,\One{t_{i-1}\le x<t_i} \mbox{ for some $t_0,t_1,\dots,t_k$ and $\theta_1,\dots,\theta_k$} \biggr\} \cap \mathcal D_{a,b}.
\end{equation*}

The proof proceeds by comparing the EA algorithm to a monotone histogram approximation of the constrained Grenander estimator. We summarize the key steps below.

\medskip
\noindent\textbf{Step~1: a monotone histogram approximation of $\tilde f^{\mle}_{n,a,b}$ with controlled height variation:}\hspace{0.5em}
We begin by noting that, as in the unconstrained case, the constrained variant $\tilde f^{\mle}_{n,a,b}$ is also a monotone histogram. For the sake of completion, we state and prove this structural property in Lemma~\ref{lem:bounded_grenander_histogram}.

However, there is no immediate control on the heights or breakpoints of $\tilde f^{\mle}_{n,a,b}$. Therefore, next, we establish such a control. Specifically, we show that $\tilde f^{\mle}_{n,a,b}$ can be approximated by a monotone histogram with at most $k$ bins, while incurring only an additional cost of $\mathrm{O}(1/k)$. We also show that despite using a compressed representation with at most $k$ bins,  adjacent bin heights are well separated on the logarithmic scale and the induced distortion in $\mathcal{L}(\cdot,n)$ is small. The following lemma formally states this.

\begin{lemma}\label{lem:compressed_version_of_constrained_grenander}
Let $\tilde f^{\mle}_{n,a,b}$ be the constrained Grenander estimator in
\eqref{eq:constrained_offline_grenander}. For any integer $k\ge1$, there exists a monotone histogram $h\in\bigcup_{r=1}^k \mathrm{Hist}_{r,\downarrow}$ of the form
\[
h(x)=\sum_{i=1}^L \theta_i\,\One{t_{i-1}\le x<t_i},\qquad 0=t_0<t_1<\cdots<t_L=1, \qquad L\le k,
\]
such that $t_{L-1}\le \max_{1\le i\le n}X_i$ and, for every $j\in\{1,\ldots,L-1\}$,
\begin{equation}\label{eq:height_drop_of_compressed_mle}
\log\theta_j-\log\theta_{j+1}\ge \frac{V}{k}.
\end{equation}
Moreover, $h$ satisfies
\begin{equation}\label{eq:logloss_distortion_of_compressed_mle}
\bigl|
\mathcal L(\tilde f^{\mle}_{n,a,b},n)-\mathcal L(h,n)\bigr|\le 2n\,\frac{V}{k}.
\end{equation}
\end{lemma}
We denote by $\tilde f^{\mle}_{n,(k),a,b}$, the compressed histogram returned by Lemma~\ref{lem:compressed_version_of_constrained_grenander}, with breakpoints $0=t_0<t_1<\cdots<t_L=1$ and heights $\theta_1\ge\cdots\ge\theta_L$. This approximation serves as the offline benchmark against which the EA algorithm is compared in the remainder of the proof.

\medskip
\noindent\textbf{Step~2: Discretizing the compressed estimator into the expert class.}\hspace{0.5em}
We now show that the compressed histogram $\tilde f^{\mle}_{n,(k),a,b}$ constructed in Step~1 can be closely approximated by an element of the expert class $\mathcal E_{k,n,\beta}$, with only a small additional distortion in $\mathcal{L}(\cdot,n)$.

The key ingredient is the following lemma, which shows that any monotone histogram with well-separated log-heights can be rounded to the expert grid without significantly affecting $\mathcal{L}(\cdot,n)$. While this may seem to follow from a straightforward rounding of endpoints and heights, the proof requires careful arguments to preserve feasibility, monotonicity, and continued membership in $\Dcal_{a,b}$.

\begin{lemma}\label{lem:discretize_hist_into_net}
Fix integers $n\ge1$ and $k\ge2$. Let $f\in\bigcup_{r=1}^k\mathrm{Hist}_{r,\downarrow}$ be of the
form
\[
f(x)=\sum_{j=1}^r \theta_j\,\One{t_{j-1}\le x<t_j},
\qquad 0=t_0<t_1<\cdots<t_r=1,\qquad r\le k,
\]
and suppose that $t_{r-1}\le 1-n^{-\gamma}$ and $\log\theta_{r-1}-\log\theta_r\ge V/k$. Then there exists a histogram $g\in\mathcal E_{k,n,\beta}$ whenever
\begin{equation}\label{eq:monotonicity_condition_lemma}
\frac{V}{k}<1,\qquad V\Delta_b<1,\qquad
a\frac{V}{k} \ \ge\ b\Delta_bV+b(\Delta_b V+2k\Delta_b)n^{\gamma}.
\end{equation}
Moreover, for any $(x_1,\ldots,x_n)\in\mathcal S_n(\beta,\gamma)$, we have that $\bigl|\mathcal L(f,n)-\mathcal L(g,n)\bigr|\le n\Delta_b V+(k-1)V$.
\end{lemma}
Here, we apply Lemma~\ref{lem:discretize_hist_into_net} with $f=\tilde f^{\mle}_{n,(k),a,b}$. Note that by Lemma~\ref{lem:compressed_version_of_constrained_grenander} and the assumption that $(x_1,\ldots,x_n)\in\mathcal S_n(\beta,\gamma)$, the final breakpoint satisfies
$t_{L-1}\le \max_{1\le i\le n} x_i \le 1-n^{-\gamma}$,
and that the last two heights obey $\log\theta_{L-1}-\log\theta_L\ge \frac{V}{k}$.
Therefore, whenever \eqref{eq:monotonicity_condition_lemma} holds, there exists a histogram $g^\star_{n,k,\beta}\in\mathcal E_{k,n,\beta}$ and
\[
\bigl|\mathcal L(\tilde f^{\mle}_{n,(k),a,b},n)-\mathcal L(g^\star_{n,k,\beta},n)\bigr|
\le
n\Delta_b V+(k-1)V.
\]

\medskip
\noindent\textbf{Step~3: bounding the size of the expert class $\mathcal E_{k,n,\beta}$.}\hspace{0.5em}
We now derive an upper bound on the cardinality of the expert class
$\mathcal E_{k,n,\beta}$ in terms of $n$, $k$, and $\beta$. By definition, the grid of admissible breakpoints satisfies
\[
|\mathcal B_n|\le n^{\beta+1}+2.
\]
The discretization of the log-heights is defined as
$\Lambda_{n,a,b}:=\log(a)+V\mathcal B_n$, so that $|\Lambda_{n,a,b}|=|\mathcal B_n|$.

We upper bound the size of $\mathcal E_{k,n,\beta}$ by ignoring the monotonicity constraint on the heights and the feasibility conditions imposed on the last height, since both can only reduce the number of admissible histograms. A histogram in $\mathcal E_{k,n,\beta}$ with at most $k$ bins is specified by at most $k-1$ internal breakpoints chosen from $\mathcal B_n$ and at most $k-1$ independent log-height values chosen from $\Lambda_{n,a,b}$. Hence,
\[
|\mathcal E_{k,n,\beta}|
\le
|\mathcal B_n|^{k-1}\,|\Lambda_{n,a,b}|^{k-1}
\;\lesssim\;
\bigl(n^{\beta+1}\bigr)^{2k-2},
\]
where the implicit constant is universal. Consequently, $\log |\mathcal E_{k,n,\beta}|\;\lesssim\;k(\beta+1)\log n$.

\medskip
\noindent\textbf{Step~4: Aggregation over the finite expert class.}\hspace{0.5em}
We now combine the previous steps and complete the proof of regret bound for EA algorithm $\hat f^{\mathrm{EA}}_{a,b}$ defined over the finite class $\mathcal E_{k,n,\beta}$. The key property we use is that exponential aggregation under log-loss performs nearly as well as the best expert in hindsight, up to a logarithmic penalty in the size of the class.

We state and prove this in Lemma~\ref{lem:mixability_of_logloss}, by which we have that $\hat f^{\mathrm{EA}}_{a,b}$ satisfies
\begin{equation}\label{eq:ea_vs_best_expert}
\mathcal L(\hat f^{\mathrm{EA}}_{a,b},n)\;\le\;
\min_{g\in\mathcal E_{k,n,\beta}} \mathcal L(g,n)\;+\;\log|\mathcal E_{k,n,\beta}|.
\end{equation}
From Step~3, we have $\log|\mathcal{E}_{k,n,\beta}|\;\lesssim\;k(\beta+1)\log n.$ Moreover, combining Steps~1 and~2, whenever the condition
\eqref{eq:monotonicity_condition_lemma} holds, there exists an expert
$g^\star_{n,k,\beta}\in\mathcal E_{k,n,\beta}$ such that
\[
\bigl|\mathcal L(\tilde f_{n,a,b}^{\mle},n)-\mathcal L(g^\star_{n,k,\beta},n)\bigr|\;\le\; 2n\frac{V}{k}+nV\Delta_b+(k-1)V.
\]
Now, we make the choice  $k \asymp \sqrt{\frac{n}{\log n}}$,
and recall that $\Delta_b=n^{-(\beta+1)}$ with $\beta>0$. With this choice, the right-hand side above is of order $\sqrt{n\log n}$, and the size of $\mathcal E_{k,n,\beta}$ also satisfies $\log |\mathcal E_{k,n,\beta}|\;\lesssim\;\sqrt{n\log n}$. Furthermore, the condition \eqref{eq:monotonicity_condition_lemma} is satisfied provided that $\gamma<\beta-\tfrac12$.
Substituting these bounds into \eqref{eq:ea_vs_best_expert}, we conclude that
\[
\mathrm{Regret}(\hat f^{\mathrm{EA}}_{a,b};n,a,b)=\mathcal L(\hat f^{\mathrm{EA}}_{a,b},n)
-\mathcal L(\tilde f_{n,a,b}^{\mle},n)\;\lesssim\;
\sqrt{n\log n},
\]
uniformly over all sequences $(x_1,\ldots,x_n)\in\mathcal S_n(\beta,\gamma)$.
This completes the proof of Theorem~\ref{thm:pathwise_regret_monotone_density}.

\section{Applying monotone density estimation to sequential testing}
\label{sec:application_to_sequential_testing}

This section shows how online monotone density estimation can be used to construct sequential tests for a fixed null hypothesis $H_0$. We begin with a brief review
of sequential testing and the role of e-values as a core tool in developing these tests. We then introduce p-to-e calibration as a way to embed classical p-value based tests into the e-value framework, and show that learning
`optimal' calibrators naturally reduces to online monotone density estimation. Finally, we establish theoretical guarantees for the resulting calibrators.

\subsection{Sequential testing: background}\label{sec:review_testing_with_evalues}

Sequential testing studies hypothesis testing with data observed sequentially, allowing decisions to be made at data-dependent times. Such problems arise naturally in modern machine learning and statistics, and have a long history in statistics \citep{barnard1946sequential,wald1992sequential,gurevich2019sequential}.

E-values have emerged as a convenient and flexible tool for constructing sequential tests in this setting \citep{wasserman2020universal, VW21, GDK24, ramdas2023game}; see \citet{ramdas2024hypothesis} for a comprehensive treatment. Formally, an e-value is a nonnegative random variable $E$ satisfying $\E[E]\le 1$ under the null hypothesis. This simple expectation constraint underlies their use in tests that remain valid under data-dependent stopping rules.

In a sequential setting, suppose $(E_t)_{t\ge1}$ is a sequence of e-values adapted to a filtration $(\mathcal F_t)$, with $\E_{H_0}[E_t\mid\mathcal F_{t-1}]\le1$ for all $t\ge1$. An associated e-process $M=(M_t)_{t\ge1}$ can be defined as
\begin{equation}\label{eq:e-martingale}
M_t=\prod_{s=1}^t E_s,\qquad M_0=1.
\end{equation}
The process $(M_t)$ is a nonnegative supermartingale under the null hypothesis. As a consequence, for any data-dependent stopping time $\tau$ satisfying
$\{\tau\le t\}\in\mathcal F_t$, we have $\E_{H_0}[M_\tau]\le 1$, and we also have $\P_{H_0}(\sup_{t\ge1}M_t\ge\frac{1}{\alpha})\le\alpha$
for all $\alpha\in(0,1]$. Therefore, thresholding $(M_t)$ at $1/\alpha$ yields a valid sequential test of $H_0$, regardless of the stopping rule. These properties form the basis for widespread use of e-values as building blocks for sequential inference.

On the other hand, most classical hypothesis testing procedures are formulated in terms of p-values and are designed for fixed-sample settings. While supported by a rich theory characterizing their validity and power in offline regimes, such tests generally fail to control Type~I error under data-dependent stopping rules. To leverage this extensive body of work in sequential settings, it is therefore natural to seek principled ways to embed p-value based methods into the e-value framework. \emph{p-to-e calibration} provides a simple and general mechanism for doing so. 

\subsection{Embedding classical tests via p-to-e calibration}
Formally, a p-to-e calibrator is a function that maps p-values to e-values. That is, for any valid p-value $P$, the transformed quantity $h(P)$ must be a valid e-value. Since this requirement must hold uniformly over all super-uniform p-values, it imposes a structural constraint on the calibrator: it must be dominated by a nonnegative and non-increasing function on $[0,1]$.  
% \RW{Slightly revised here. Also, we use ``nonnegative" without - elsewhere.}

One defines a \emph{p-to-e calibrator} as a non-increasing function $h:[0,1]\to[0,\infty)$ satisfying $\int_0^1 h(p)\,\mathrm{d}p\le 1$, equivalently, a monotone (sub-)probability Lebesgue-density function on $[0,1]$. By definition, if $P$ is a valid p-value, then $h(P)$ is a valid e-value. \citet{VW21} further show that any admissible calibrator must attain equality. Thus, we observe that
\begin{quote}
    \centering \emph{All admissible calibrators are decreasing densities supported on $[0,1]$.}
\end{quote}
We restrict our attention to the admissible calibrators, that is the class $\Dcal$. The connection between calibration and monotone densities plays a central role in our development of the calibrators; this link has appeared previously in \citet{arnold2023sequentially}. From this point onward, we refer to a p-to-e calibrator simply as a calibrator. 

We now describe how calibrators integrate with the e-value framework when only p-values are observed. Suppose we observe a sequence $(P_t)_{t\ge1}$ of p-variables produced by classical hypothesis tests, satisfying under the null hypothesis $H_0$ that
\[
\P_{H_0}(P_t\le\alpha\mid\mathcal F_{t-1})\le\alpha,
\qquad
\textnormal{for all } t\ge1,\ \alpha\in(0,1).
\]
Our goal is to convert these p-values into e-values in an online and adaptive
manner. At each time $t$, we select an $\mathcal F_{t-1}$-predictable calibrator
$h_t\in\Dcal$ and transform $P_t$ into an e-value $h_t(P_t)$. The resulting
e-process is given by
\begin{equation}\label{eq:e-martingale-p}
M_t=\prod_{s=1}^t h_s(P_s),
\qquad
M_0=1.
\end{equation}
By construction, this enables a valid sequential test for any such online calibrator.

\subsection{Empirically adaptive p-to-e calibration}\label{sec:empirical_calibrators}

While any calibrator $h$ yields a valid sequential test via \eqref{eq:e-martingale-p}, the power of the resulting procedure depends critically on how the calibrators are chosen. To formalize an optimal benchmark, \citet[Section~3.6]{ramdas2024hypothesis} introduce the following oracle procedure.

\begin{definition}[Log-optimal decreasing calibrator]\label{defn:log_optimal_calibrator}
For a given alternative distribution $\bQ$, the \emph{$\bQ$-log-optimal calibrator} $h^{\mathrm{opt}}=(h_t^{\mathrm{opt}})_{t\ge1}$ is defined by
\[
h_t^{\mathrm{opt}}\in\argmax_{h\in\Dcal}\E_\bQ\!\left[\log h(P_t)\mid\mathcal F_{t-1}\right].
\]
\end{definition}

Since the alternative distribution $\bQ$ is unknown, the oracle calibrator $h_t^{\mathrm{opt}}$ is not directly computable. This motivates the construction of empirically adaptive calibrators that approximate the oracle and achieve high power in practice.

Throughout the remainder of this section, we work under the alternative hypothesis, which posits that $P_1,P_2,\ldots\stackrel{iid}{\sim}\bQ$ with a monotone density $q\in\Dcal$. 
The i.i.d.\ model considered here serves as a natural starting point for the study of adaptive $p$-to-$e$ calibration. While practical sequential testing problems often involve dependence and non-stationarity, empirically adaptive $p$-to-$e$ calibration remains relatively underexplored even in the iid setting. Moreover, under most alternatives arising in standard testing problems, $p$-values tend to concentrate near zero and exhibit a decreasing density, aligning naturally with our modeling assumption. For instance, in one-sided Gaussian mean testing, $H_0:\mu=0$ versus $H_1:\mu>0$ with $Z\sim N(\mu,1)$, the corresponding $p$-value $\Phi(-Z)$ has a monotone decreasing density under the alternative.

In the above iid model, the log-optimal calibrator is given by $h_t^{\mathrm{opt}}=q$ for all $t\ge1$. Indeed, for any $h\in\Dcal$,
\[
\E_\bQ\!\left[\log h(P_t)\mid\mathcal F_{t-1}\right]
\le\E_\bQ\!\left[\log q(P_t)\mid\mathcal F_{t-1}\right]-\mathrm{KL}(q\|h)\le\E_\bQ\!\left[\log q(P_t)\mid\mathcal F_{t-1}\right].
\]
Consequently, we obtain the following interesting fact, which appears implicitly in prior work but has not been systematically developed:
\begin{quote}
    \centering \emph{
    % All calibrators  are (sub-)probability Lebesgue-density functions on $[0,1]$. Admissible calibrators are decreasing densities. 
    The log-optimal calibrator is given by $q$. Therefore, approximating the log-optimal calibrator from the observed $p$-values reduces to the problem of online monotone density estimation.}
\end{quote}

We therefore use the two online monotone density estimators introduced in Section~\ref{sec:online_monotone_density_estimation} as empirically adaptive calibrators:
\begin{itemize}[itemsep=0.7pt, topsep=1.2pt]
    \item \textbf{OG calibrator.}
    The estimator $\hat{f}^{\,\mathrm{OG}}_{a,b}$ yields a natural adaptive calibrator. Its online likelihood maximization can be interpreted as empirical log-optimal calibration, where the conditional distribution of the p-value $P_t$ under $\bQ$ is replaced by the empirical distribution of past observations, $(t-1)^{-1}\sum_{s=1}^{t-1}\delta_{P_s}$.
    \item \textbf{EA calibrator.}
    Alternatively, $\hat{f}^{\,\mathrm{EA}}_{a,b}$ defines an adaptive calibrator by exponentially weighting a finite collection of expert calibrators, with
    weights updated according to accumulated log-wealth from past $p$-values.
\end{itemize}

Both $\hat{f}^{\,\mathrm{OG}}_{a,b}$ and $\hat{f}^{\,\mathrm{EA}}_{a,b}$ yield valid $e$-processes via \eqref{eq:e-martingale-p}, and therefore induce valid sequential tests. In practice, a natural choice is to take $a>0$ and $b=\infty$ to ensure that calibrators never output zero while allowing unbounded values to enable strong evidence accumulation.

\subsection{Optimality of empirically adaptive calibrators}

We now study optimality properties of these empirically adaptive calibrators, establishing their effectiveness as practical procedures for sequential testing.
In particular, we show that both the OG and EA algorithms are asymptotically log-optimal calibrators.
\begin{theorem}[Asymptotic log-optimality]\label{thm:asymptotic_log_optimality}
Suppose $P_1,P_2,\ldots\stackrel{iid}{\sim}\bQ$, where $q\in\Dcal_{a,b}$ for some $0<a<b<\infty$. Then for any calibrator $\hat h\in\{\hat{f}^{\,\mathrm{OG}}_{a,b},\hat{f}^{\,\mathrm{EA}}_{a,b}\}$,
\[
\lim_{n\to\infty}\frac{1}{n}\Bigl(\log M_n(\hat h)-\log M_n(h^{\mathrm{opt}})\Bigr)=0\qquad \bQ\textnormal{-almost surely},
\]
where $h^{\mathrm{opt}}$ denotes the log-optimal calibrator from Definition~\ref{defn:log_optimal_calibrator}. 
% \RW{I think you meant Definition 7}Thanks!!
\end{theorem}
Here, by   $\hat f^{\mathrm{EA}}_{a,b}$ we refer to the EA calibrator based on the set of experts $\mathcal E_{k,n,\beta}$ with  $k=\lfloor(n/\log{n})^{1/2}\rfloor$. Theorem~\ref{thm:asymptotic_log_optimality} shows that the proposed calibrators asymptotically achieve the same per-sample log-wealth as the optimal calibrator tailored to the true distribution $\bQ$. As a consequence, for both procedures, the accumulated evidence diverges under any non-uniform monotone alternative.
% \RW{How is the EA expert set specified? I think the optimality  result relies on it. I guess we missed a condition like the one in Th2.} Thanks for reminding! 

% \RW{Question, but we don't need to answer it now: Does this hold true under the same condition $q\in\Dcal_{a,b}$  but $\hat h\in\{\hat{f}^{\,\mathrm{OG}}_{0,\infty},\hat{f}^{\,\mathrm{EA}}_{0,\infty}\}$?} RH: as per my understanding, yes! We have to prove an additional lemma which will in spirit say that if the true density q is bounded above and below, then the Grenander estimator will also be bounded away and below. One small caveat is the bounds on the classical Grenander would look like a-$\delta$ and $b+\delta$ with high probability. So ideally, for EA calibrator..we have to give a little slack while defining the expert class. This is all assuming that the Lemma I am hypothesizing is actually true!

\begin{corollary}[Almost sure divergence of log-wealth]\label{cor:consistency_calibrators}
Suppose $P_1,P_2,\ldots\stackrel{iid}{\sim}\bQ$, where
$q\in\Dcal_{a,b}$ for some $0<a<b<\infty$ and $\int_0^1 q(u)\log q(u)\,\d u>0$.
Then for any $\hat h\in\{\hat{f}^{\,\mathrm{OG}}_{a,b},\hat{f}^{\,\mathrm{EA}}_{a,b}\}$,
\[
\lim_{n\to\infty}\frac{1}{n}\log M_n(\hat h)=\int_0^1 q(u)\log q(u)\,\d u
\qquad \bQ\textnormal{-almost surely}.
\]
In particular, $\log M_n(\hat h)\to\infty$ $\bQ$-almost surely as $n\to\infty$.
\end{corollary}
We observe that the condition $\mathrm{KL}(q\|f_U)=\int_0^1 q(u)\log q(u)\,\d u>0$, where $f_U$ denotes the density of $\mathrm{Unif}[0,1]$, simply requires $\bQ\not\overset{d}{=} \mathrm{Unif}[0,1]$. Under such alternatives, we do expect the log-wealth to diverge and the associated sequential test to have a finite detection time almost surely.
% \RW{Theorem \ref{thm:asymptotic_log_optimality}
% and Corollary \ref{cor:consistency_calibrators} are connected to the log-optimality of the empirically adaptive e-process; see \citet[Theorem 7.22]{ramdas2024hypothesis}
% and \citet[Theorem 3]{wang2025backtesting}.
% In the latter two results, e-processes are constructed based on e-values as in \eqref{eq:e-martingale} instead of p-values as in \eqref{eq:e-martingale-p}. 
% }\AR{I would rather not say too much about it!}

\section{Numerical experiments}\label{sec:simulations}
Now, we study the empirical performance of the proposed online estimators through a series of numerical experiments.

\subsection{Empirical evaluation of OG and EA in online monotone density estimation }\label{sec:empirical_study}

We compare the cumulative log-likelihood trajectories of the OG and EA algorithms against two benchmarks: the true density $q$ and the offline Grenander estimator $\tilde f_n^{\mle}$ computed in hindsight. We consider sample paths of length $n=1000$ and report averages over $B=50$ independent replications. For the EA procedure, we use a structured collection of monotone histogram experts with $2$--$4$ breakpoints chosen from a coarse grid on $(0.1,0.2,\ldots,0.9)$ and geometrically decaying heights, normalized to integrate to one. Although substantially smaller than the theoretical construction in Section~\ref{sec:bounding_regret_for_EA}, this expert class remains sufficiently expressive in practice and performs well empirically.

\begin{figure}[t]
    \centering
    \includegraphics[width=0.92\linewidth]{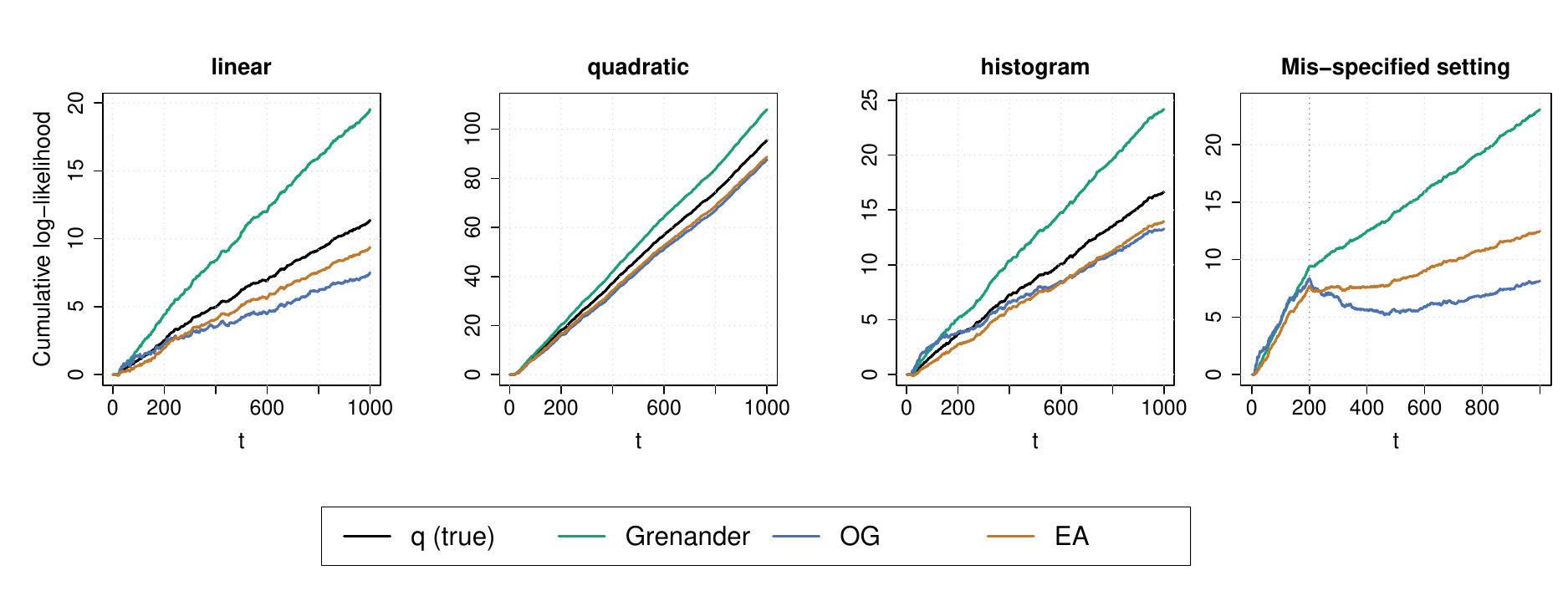}
    \caption{Cumulative log-likelihood trajectories in well-specified and mis-specified stochastic settings.}
    \label{fig:combined_fig}
\end{figure}

\paragraph{Well-specified stochastic setting.}
We start by considering the well-specified regime, where $X_1,X_2,\ldots \stackrel{iid}{\sim} \bQ$ with a monotone density $q\in\Dcal$. We examine three representative models: (1) a linear density $q(u)=5/4-u/2$, (2) a quadratic density $q(u)=3(1-u)^2$, and (3) a piecewise-constant density on four equal-width bins with heights $5/4,\ 13/12,\ 11/12,$ and $3/4$.

Figure~\ref{fig:combined_fig} reports the resulting cumulative log-likelihood trajectories. As expected, the offline Grenander estimator is optimistic and attains the largest cumulative likelihood. Both online procedures closely track the true density, while EA typically exhibits smaller regret relative to the offline Grenander benchmark, indicating improved finite-sample adaptivity. This difference is most pronounced in the linear model, where the monotone structure is relatively mild; in the remaining settings, the two methods behave similarly.

\paragraph{Mis-specified setting.}
We next consider a mis-specified scenario in which the data-generating density changes over time. Specifically, observations are initially drawn iid from the linear monotone density $q(u)=\delta_1-(\delta_1-\delta_0)u$ with $(\delta_0,\delta_1)=(0.5,1.5)$ up to a change point $\tau=200$, after which the parameters switch to $(\delta_0,\delta_1)=(0.75,1.25)$.

The rightmost panel of Figure~\ref{fig:combined_fig} illustrates the resulting cumulative log-likelihood trajectories. In this setting, the adaptive nature of EA becomes more evident: by reweighting experts according to recent observations, EA adjusts more rapidly to the evolving density and incurs substantially smaller regret. In contrast, OG is more strongly influenced by pre-change observations and therefore adapts more slowly following the shift. This experiment highlights the advantage of expert aggregation when the monotone or iid assumptions are only approximately satisfied.

\subsection{Sequential changepoint detection}
We now study the performance of the OG and EA calibrators in a sequential changepoint detection problem. We simulate a sequence $X_1,X_2,\ldots,X_n$ with $n=2000$, where
\[
X_t\sim N(0,1)\quad \textnormal{for } t\le \tau,
\qquad
X_t\sim N(\mu,1)\quad \textnormal{for } t>\tau,
\]
with changepoint $\tau=500$. At each time step $t\in[n]$, we construct conformal $p$-values as
\[
p_t=\frac{1}{t}\left(\sum_{i<t} \One{X_i>X_t}
+U_t\sum_{i=1}^t\One{X_i=X_t}\right),
\qquad
(U_1,\ldots,U_n)\stackrel{iid}{\sim}\mathrm{Unif}[0,1].
\]
This provides one of the simplest conformal $p$-value constructions in this setting. Observe that for $t\le\tau$, the sequence $(X_1,\ldots,X_t)$ is exchangeable, and hence $p_t\sim \mathrm{Unif}[0,1]$. After the changepoint, however, the observations become systematically larger; this forces the conformal $p$-values to concentrate near zero and, marginally, exhibit a near-decreasing density.

The resulting sequence of conformal $p$-values is incorporated into~\eqref{eq:e-martingale-p} to construct sequential test martingales. A classical approach in the conformal martingale literature is to use a fixed calibrator, that is, $h_s=h$ for all $s\in[n]$; see \citet{vovk2003testing,volkhonskiy2017inductive}. Following \citet[Section~3.6]{volkhonskiy2017inductive}, we consider the following two fixed baseline calibrators:
\begin{itemize}[itemsep=0.8pt, topsep=1.2pt]
    \item \textbf{Mixture-power calibrator:}
    \[
    h_{\mathrm{mix}}(p)=\int_0^1 \varepsilon p^{\varepsilon-1}\,\mathrm{d}\varepsilon,
    \]
    which corresponds to a continuous mixture of power betting functions.
    
    \item \textbf{Constant-step calibrator:}
    \[
    h_{\mathrm{step}}(p)=\begin{cases}
        1.1 & \text{if~} p>0.5\\
        0.9 & \text{otherwise}.
    \end{cases}
    \]
\end{itemize}
We compare these two fixed-calibrator baselines with the adaptive OG and EA calibrators. For the EA procedure, we use the same expert class as in Section~\ref{sec:empirical_study}.

We consider three choices of the post-change mean, namely $\mu\in\{0.5,1,2\}$, corresponding respectively to weak, moderate, and strong mean shifts, thereby ranging from harder to easier changepoint detection problems. For each setting, we generate $100$ independent repetitions and report the average cumulative log-wealth trajectories of the competing test martingales in Figure~\ref{fig:changepoint_detection}

\begin{figure}[t] 
\centering 
\includegraphics[width=0.7\linewidth]{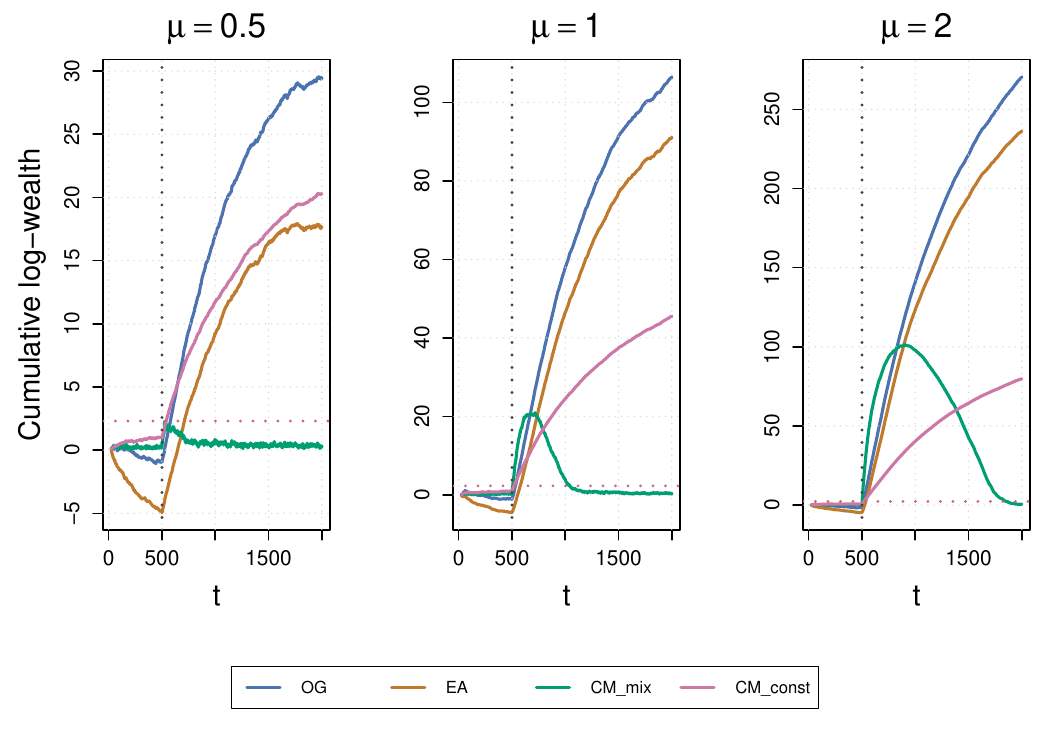} 
\caption{Average cumulative log-wealth trajectories for sequential changepoint detection under weak ($\mu=0.5$), moderate ($\mu=1$), and strong ($\mu=2$) mean shifts, shown from left to right. The vertical dashed line marks the changepoint location $\tau=500$, while the horizontal red line denotes the rejection threshold $\log(1/\alpha)$ with $\alpha=0.1$.} \label{fig:changepoint_detection} 
\end{figure}

\paragraph{Discussion on the results.}
Across all settings, both the OG and EA calibrators perform robustly, and the corresponding test martingales eventually cross the rejection threshold, correctly detecting the change in the data-generating process. As expected, weaker shifts require substantially longer detection times. 

In contrast, the performance of the fixed calibrators heavily depends on the signal strength. In the strong-shift regime ($\mu=2$), the post-change $p$-values become extremely small, which is effectively exploited by the mixture-power calibrator (denoted as \textnormal{CM\_mix} in Figure~\ref{fig:changepoint_detection}), making it competitive with the OG and EA procedures and substantially stronger than the constant-step calibrator. On the other hand, in the weak-shift regime ($\mu=0.5$), the post-change $p$-values deviate only mildly from uniformity. In this case, the constant-step calibrator (denoted as \textnormal{CM\_const} in Figure~\ref{fig:changepoint_detection}) performs more favorably than the mixture-power alternative.

We also observe that, since the sequence length is relatively large ($n=2000$) compared to the changepoint location ($\tau=500$), the conformal $p$-values gradually become closer to uniform at later time points. Consequently, the corresponding test martingales may lose wealth after an initial growth phase. This effect is particularly pronounced for the mixture-power calibrator.

\section{Discussion}
In this work, we introduced two online algorithms for monotone density estimation. In the well-specified stochastic setting, where the true density $q$ is monotone, we derived bounds on the excess KL-risk of both procedures. We also obtained a pathwise regret bound for the EA algorithm, comparing its performance to that of the offline Grenander estimator computed in hindsight.

We further showed that the problem of p-to-e calibration in sequential testing can be naturally formulated as an instance of online monotone density estimation. This connection allowed us to adapt the proposed online algorithms as data-driven calibrators and to establish their optimality.

A promising direction for future work is to analyze the behavior of these empirically adaptive calibrators under non-iid alternatives, which would further enhance their applicability in practical sequential testing settings.

\section*{Acknowledgments}
Rohan Hore and Aaditya Ramdas acknowledge the funding from the Sloan Fellowship.
Ruodu Wang is supported by the Natural Sciences and Engineering Research Council of Canada (CRC-2022-00141, RGPIN-2024-03728). The authors also thank Aditya Guntuboyina for pointers to relevant literature.

\bibliographystyle{chicago}
\bibliography{reference}

\appendix
\section{Proofs}

\subsection{Proof of results from Section~\ref{sec:online_monotone_density_estimation}}
\subsubsection{Proof of Lemma~\ref{lem:KL_rate_of_montone_mle}}

Fix $0<a<b<\infty$ throughout the proof. We write $N_{[\,]}(u,\Dcal_{a,b},\hel)$ for the bracketing entropy of $\Dcal_{a,b}$ with respect to the Hellinger distance. By \citet[Theorem~1]{wong1995probability}, there exist universal constants $c_1,c_2>0$ such that for any $\varepsilon>0$ satisfying
\begin{equation}\label{eq:WS_entropy_condition}
\int_0^\varepsilon \sqrt{N_{[\,]}(u,\Dcal_{a,b},\hel)}\,\d u\;\le\;c_1\sqrt n\,\varepsilon^2,
\end{equation}
we have $\P_{\bQ}(\mathcal A)\le2\exp(-c_2 n\varepsilon^2)$, where
\begin{equation}\label{eq:WS_LR_bound}
\mathcal A:=\Biggl\{\sup_{f\in\Dcal_{a,b}:\,\hel(f,q)\ge\varepsilon}\sum_{s=1}^n\log\frac{f(X_s)}{q(X_s)}\ge-\,c_2 n\varepsilon^2\Biggr\}.
\end{equation}
By definition of the MLE,
\[
\sum_{s=1}^n \log \tilde{f}_{n,a,b}^\mle(X_s)\ge\sum_{s=1}^n \log q(X_s),
\]
and hence
\[
\sum_{s=1}^n\log\frac{\tilde{f}_{n,a,b}^\mle(X_s)}{q(X_s)}\ge 0.
\]
On the event $\mathcal A^c$, any density $f$ satisfying $\hel(f,q)\ge\varepsilon$ obeys
\[
\sum_{s=1}^n \log\frac{f(X_s)}{q(X_s)}<-\,c_2 n\varepsilon^2<0,
\]
and therefore cannot coincide with the MLE $\tilde{f}_{n,a,b}^\mle$. Consequently,
\[
\mathcal A^c\subseteq\bigl\{\hel(\tilde{f}_{n,a,b}^\mle,q<\varepsilon\bigr\}.
\]
Thus, for any $\varepsilon>0$ satisfying~\eqref{eq:WS_entropy_condition},
\[
\P_{\bQ}\!\left(\hel(\tilde{f}_{n,a,b}^\mle,q)\ge\varepsilon\right)
\le 2\exp(-c_2 n\varepsilon^2).
\]
By \citet{gao2007entropy}, the bracketing entropy of $\Dcal_{a,b}$ in Hellinger distance satisfies
\[
N_{[\,]}(u,\Dcal_{a,b},\hel)\le C\,u^{-1},
\]
for a constant $C>0$ depending only on $a$ and $b$.
Therefore,
\[
\int_0^\varepsilon \sqrt{N_{[\,]}(u,\Dcal_{a,b},\hel)}\,\d u \le C\int_0^\varepsilon u^{-1/2}\,\d u=2C\,\varepsilon^{1/2}.
\]
The entropy condition~\eqref{eq:WS_entropy_condition} thus holds for all
$\varepsilon\ge\varepsilon_n$, where $\varepsilon_n\asymp n^{-1/3}$.
It follows that
\[
\E_{\bQ}\!\left[\hel^2(\tilde{f}_{n,a,b}^\mle,q)\right]\le\varepsilon_n^2+\int_{\varepsilon_n^2}^1 2\exp(-c_2 n s)\,\d s
\le\varepsilon_n^2+\frac{2}{c_2 n}\le C'(a,b)\,n^{-2/3},
\]
for an appropriate constant $C'(a,b)>0$.
Finally, by Lemma~\ref{lem:L2_H_KL_equivalence}, since
$q,\tilde{f}_{n,a,b}^\mle\in\Dcal_{a,b}$,
\[
\mathrm{KL}\bigl(q\|\tilde{f}_{n,a,b}^\mle\bigr)\le C''(a,b)\, \hel^2\bigl(q,\tilde{f}_{n,a,b}^\mle\bigr)
\]
for another constant $C''(a,b)>0$.
Taking expectations yields
\[
\E_\bQ\!\left[\mathrm{KL}\bigl(q\|\tilde{f}_{n,a,b}^\mle\bigr)\right]
\le \kappa(a,b)\,n^{-2/3},
\]
for some constant $\kappa(a,b)\in(0,\infty)$, as required.
\hfill$\blacksquare$

\subsubsection{Proof of Theorem~\ref{thm:excess_KL_risk_OG_and_EA}: bounds on excess KL-risk}

Fix $n\ge2$ and $0<a<b<\infty$ throughout.

\medskip
\noindent\textbf{Bounding the risk of the OG estimator.}
We first establish the risk bound for the online Grenander (OG) estimator.
For any $t\ge2$, note that
\[
\hat{f}^{\,\mathrm{OG}}_{t,a,b}=\tilde{f}^{\mle}_{t-1,a,b}.
\]
Recall that 
\[
\mathrm{Risk}_\bQ(\hat f;n)
=\sum_{s=1}^n\E\bigl[\mathrm{KL}(q\|\hat f_s)\bigr].
\]
Thus, by Lemma~\ref{lem:KL_rate_of_montone_mle}, for any $q\in\Dcal_{a,b}$,
\begin{align*}
\mathrm{Risk}_\bQ(\hat{f}^{\,\mathrm{OG}}_{a,b};n)
&=\sum_{s=1}^n\E_\bQ\left[\mathrm{KL}\bigl(q\|\hat{f}^{\,\mathrm{OG}}_{s,a,b}\bigr)\right] \\
&=\sum_{s=1}^n\E_\bQ\left[\mathrm{KL}\bigl(q\|\tilde{f}^{\mle}_{s-1,a,b}\bigr)\right] \\
&\le\kappa_{a,b}\sum_{s=1}^n (s-1)^{-2/3}\;\le\;\Gamma_{\mathrm{OG}}(a,b)\,n^{1/3},
\end{align*}
for some finite constant $\Gamma_{\mathrm{OG}}(a,b)>0$.
This completes the proof for the OG estimator.

\medskip
\noindent\textbf{Bounding the risk of the EA estimator.}
Fix $q\in\Dcal_{a,b}$.
Let
\[
\mathcal N(\varepsilon,\Dcal_{a,b},\hel)\subset\Dcal_{a,b}
\]
be an $\varepsilon$-net of $\Dcal_{a,b}$ with respect to the Hellinger distance $\hel$.
Since $q\in\Dcal_{a,b}$, there exists $q^\star\in\mathcal N(\varepsilon,\Dcal_{a,b},\hel)$ such that
$\hel(q,q^\star)\le\varepsilon$.
By slight abuse of notation, we write $q^\star\equiv(q^\star)_{t\ge1}$.
By definition,
\[
\mathrm{Risk}_\bQ(q^\star;n)=n\,\mathrm{KL}(q\|q^\star).
\]
By Lemma~\ref{lem:L2_H_KL_equivalence}, there exists a constant $C_1(a,b)>0$ such that
\[
\mathrm{KL}(q\|q^\star)\le C_1(a,b)\,\hel^2(q,q^\star) \le C_1(a,b)\,\varepsilon^2.
\]
Hence,
\begin{equation}\label{eq:q_star_risk}
    \mathrm{Risk}_\bQ(q^\star;n)\le C_1(a,b)\,n\varepsilon^2
\end{equation}

Next, we enumerate the elements of $\mathcal N(\varepsilon,\Dcal_{a,b},\hel)$ as
$\{g_1,\ldots,g_{m_\varepsilon}\}$, and consider the exponentially weighted estimator over this finite class:
\[
\hat{f}^{\mathrm{EA}}_{t,a,b}(x)=\sum_{g\in\mathcal N(\varepsilon,\Dcal_{a,b},\hel)}w_{t-1}(g)\,g(x),\qquad
w_{t-1}(g)\propto\frac{1}{m_\varepsilon}\prod_{s=1}^{t-1} g(X_s).
\]
By Lemma~\ref{lem:mixability_of_logloss}, we have that
\[
\sum_{s=1}^n \log \hat{f}^{\mathrm{EA}}_{s,a,b}(X_s)
\ge\max_{g\in\mathcal N(\varepsilon,\Dcal_{a,b},\hel)}\sum_{s=1}^n \log g(X_s)-\log m_\varepsilon.
\]
In particular, taking $g=q^\star$ and expectation over $X_1,\ldots,X_n\iidsim\bQ$, we obtain
\begin{equation}\label{eq:online_predictor_risk}
\E_\bQ\left[\sum_{s=1}^n \log q^\star(X_s)-\sum_{s=1}^n \log \hat{f}^{\mathrm{EA}}_{s,a,b}(X_s)\right]\le\log m_\varepsilon.
\end{equation}

Finally, by \citet[Theorem~1.1]{gao2007entropy}, the class of bounded monotone densities on $[0,1]$ satisfies
\[
\log N(\varepsilon,\Dcal,\|\cdot\|_2)\le C_2\,\varepsilon^{-1}
\]
for some constant $C_2>0$.
Since $\Dcal_{a,b}\subset\Dcal$, the same bound holds for $\Dcal_{a,b}$.
Moreover, by Lemma~\ref{lem:L2_H_KL_equivalence}, the $L_2$ norm and the Hellinger distance are equivalent on $\Dcal_{a,b}$, yielding
\[
\log m_\epsilon=\log N(\varepsilon,\Dcal_{a,b},\hel)\le C_2\,\varepsilon^{-1}.
\]
Combining \eqref{eq:q_star_risk} and \eqref{eq:online_predictor_risk}, we obtain
\[
\mathrm{Risk}_\bQ(\hat{f}^{\,\mathrm{EA}}_{a,b};n)
\le C(a,b) \left(n\varepsilon^2+\frac{1}{\varepsilon}\right),
\]
for a constant $C(a,b)>0$ depending only on $a$ and $b$.
Choosing $\varepsilon=n^{-1/3}$ yields the desired bound.

\hfill$\blacksquare$

\subsection{Completing the proof of Theorem~\ref{thm:pathwise_regret_monotone_density}: pathwise regret bound}\label{app:completing_proof_of_pathwise_regret}

\subsubsection{Proof of Lemma~\ref{lem:compressed_version_of_constrained_grenander}}
Fix $k\ge 1$ and set $V:=\log(b/a)$.
By Lemma~\ref{lem:bounded_grenander_histogram},
$\tilde f_{n,a,b}^{\mle}$ admits a histogram representation:
there exists $m_n\ge 1$ such that
\[
\tilde f_{n,a,b}^{\mle}(x)
=\sum_{i=1}^{m_n} w_i\,\One{t_{i-1}\le x<t_i},
\qquad
0=t_0<t_1<\cdots<t_{m_n}=1,
\]
where $b\ge w_1\ge\cdots\ge w_{m_n}\ge a$ and $t_{m_n-1}\le X_{(n)}$. 

Now, in order to define the compressed histogram, we start by setting $i'_0=0$ and, for $\ell\ge 1$,
\[
i'_\ell:=\max\Bigl\{
j\in\{i'_{\ell-1}+1,\ldots,m_n\}:
\log w_{i'_{\ell-1}+1}-\log w_j \le V/k
\Bigr\}.
\]
If $i'_\ell=m_n$, the recursion terminates. Let $L$ be the resulting number of blocks.
Note that each time a new block begins, the log-height drops by at least $V/k$, while the total variation of $\log \tilde f_{n,a,b}^{\mle}$ is at most $V$. Hence $L\le k$.

For each block $\ell=1,\ldots,L$, define the block interval
\[
[t'_{\ell-1},t'_\ell):=[t_{i'_{\ell-1}},t_{i'_\ell}),
\qquad\text{so that}\qquad
0=t'_0<t'_1<\cdots<t'_L=1.
\]
Since $i'_L=m_n$, we have $t'_{L-1}=t_{i'_{L-1}}\le t_{m_n-1}\le X_{(n)}$.

Now, for the compressed histogram, we assign to the bin $[t'_{\ell-1},t'_\ell)$ a height proportional to the left-end height of each block, i.e.
\[
\theta_\ell =\alpha\, w_{i'_{\ell-1}+1},
\qquad \ell=1,\ldots,L,
\]
for a suitable $\alpha>0$, and we set
\[
h(x):=\sum_{\ell=1}^L \theta_\ell\,\One{t'_{\ell-1}\le x<t'_\ell}.
\]
Observe that we can choose $\alpha>0$ suitably, so that the $h(x)$ integrates to $1$, i.e.,
\[
\alpha\, \sum_{\ell=1}(t'_{\ell}-t'_{\ell-1}) w_{i'_{\ell-1}+1}=1
\]
Observe that $\ell\in\{1,\ldots,L-1\}$, by the definition of $\{i'_\ell\}_{\ell\ge 1}$, we have
$\log w_{i'_{\ell-1}+1}-\log w_{i_\ell'+1} > V/k$, and therefore
\[
\log \bar\theta_\ell-\log \bar\theta_{\ell+1}=\log w_{i'_{\ell-1}+1}-\log w_{i_\ell'+1}\ge \frac{V}{k}.
\]
This proves the first part.

For the second part, fix any $x\in[t'_{\ell-1},t'_\ell)$ for some $\ell$. Then , $\log w_{i'_{\ell-1}+1}-\log w_j
\le \frac{V}{k}$ implies that  
\[
\log(h(x)/\alpha)-\log \tilde f_{n,a,b}^{\mle}(x) \le V/k.
\]
Equivalently,
\[
 h(x)/\alpha = w_{i'_{\ell-1}+1}\le w_j e^{V/k}=\tilde{f}_{n,a,b}^{\mle}(x)e^{V/k}.
\]
This holds for any $\ell\in \{1,\ldots,L\}$ and any $x\in[t'_{\ell-1},t'_\ell)$. Hence
\[
\frac{1}{\alpha}=\int_0^1 h(u) du\le e^{V/k}\int_0^1 \tilde{f}_{n,a,b}^{\mle}(u) du= e^{V/k}
\]
Hence, 
\[
|\log \alpha|\le \frac{V}{k}.
\]
Therefore, for all $x\in[0,1]$,
\[
\bigl|\log \tilde f_{n,a,b}^{\mle}(x)-\log h(x)\bigr|
=
\bigl|\log \tilde f_{n,a,b}^{\mle}(x)-\log(h(x)/\alpha)+\log
(\alpha) \bigr|
\le \frac{2V}{k}.
\]

Finally, for the observed sample path $(x_1,\ldots,x_n)$,
\[
\bigl|\mathcal L(\tilde f_{n,a,b}^{\mle},n)-\mathcal L(h,n)\bigr|
=\left|\sum_{t=1}^n
\log \tilde f_{n,a,b}^{\mle}(x_t)-\log h(x_t)\right|\le\frac{2nV}{k}.
\]
This completes the proof.
\hfill $\blacksquare$
\subsubsection{Proof of Lemma~\ref{lem:discretize_hist_into_net}}
\noindent \textbf{Proof of first part:}
Given a $f$ 
\[
f(x)=\sum_{j=1}^r \theta_j\,\One{t_{j-1}\le x<t_j},
\qquad 0=t_0<t_1<\cdots<t_r=1,\qquad r\le k,
\]
as in the theorem hypothesis, we construct $g\in \mathcal{E}_{k,n,\beta}$ as follows whenever \eqref{eq:monotonicity_condition_lemma} holds:
\\[0.5em]
\noindent \textbf{\emph{Endpoint approximation:}}
First, for each $j=1,\ldots,r-1$, define the left-rounded endpoints
\[
t_j' := \max\{u\in\mathcal B_n:\ u\le t_j\}.
\]
Then $t_j-\Delta_b\le t_j'\le t_j$. Set $t_0'=0$ and $t_r'=1$.
\\[0.5em]
\noindent \textbf{\emph{Log-height approximation:}}
Next, for $j=1,\ldots,r-1$, define the down-rounded log-heights
\[
\ell_j' := \max\{\ell\in\Lambda_{n,a,b}:\ \ell\le \log\theta_j\},
\qquad \theta_j' := e^{\ell_j'}.
\]
We write $\eta:=V\Delta_b$.
By construction, $\log\theta_j-\eta\le \ell_j'\le \log\theta_j$. By \eqref{eq:monotonicity_condition_lemma}, $\eta<1$ and
\begin{equation}\label{eq:closeness_of_theta_and_theta'}
    \theta_j\ge \theta_j'\ge e^{-\eta}\theta_j
\qquad\text{and}\qquad
\theta_j-\theta_j'\le (1-e^{-\eta})\theta_j\le b\eta.
\end{equation}
Moreover, since the map $\ell\mapsto \max\{\ell'\in\Lambda_{n,a,b}:\ell'\le \ell\}$ is monotone,
\[
b\ge \theta_1'\ge\cdots\ge \theta_{r-1}'\ge a.
\]

\noindent\textbf{\emph{Final height assignment:}}
Finally, we define $\theta_r'$ by normalization:
\[
\theta_r'
:=
\frac{1-\sum_{j=1}^{r-1}\theta_j'(t_j'-t_{j-1}')}{1-t_{r-1}'}.
\]
Consequently, we define the histogram
\[
g(x):=\sum_{j=1}^r \theta_j'\,\One{t_{j-1}'\le x<t_j'}.
\]
In order to show that $g\in \mathcal{E}_{k,n,\beta}$ we need to show that $\theta'_r\ge a$ and $\theta'_{r-1}\ge \theta'_r$.

We first show that $\theta_r'\ge \theta_r\ge a$. Suppose for a contradiction that $\theta_r'< \theta_r$.
Since $t_j'\le t_j$ and $\theta_j'\le \theta_j$ for all $j\le r-1$, the left-rounded partition and down-rounded heights imply the pointwise domination $g(x)\le f(x)$ for all $x\in[0,1]$.
In particular,
\[
\sum_{j=1}^{r}\theta_j'(t_j'-t_{j-1}')
=
\int_0^{1} g(x)\,dx
<
\int_0^{1} f(x)\,dx
=1,
\]
This is a contradiction, since $\int_0^{1} g(x)\,dx=1$ by definition.
This proves that $\theta_r'\ge a$.

To prove that $\theta'_{r-1}\ge \theta'_r$, we start by upper bounding  $\theta_r'$. Using normalization for $f$ and $g$, we may write
\[
\theta_r'
=
\frac{\theta_r(1-t_{r-1})}{1-t_{r-1}'}
+
\frac{\sum_{j=1}^{r-1}\theta_j(t_j-t_{j-1})-\sum_{j=1}^{r-1}\theta_j'(t_j'-t_{j-1}')}{1-t_{r-1}'}.
\]
Since $t_{r-1}'\le t_{r-1}$, the first term is at most $\theta_r$.
For the second term, we bound the numerator by
\begin{align*}
&\sum_{j=1}^{r-1}\theta_j(t_j-t_{j-1})-\sum_{j=1}^{r-1}\theta_j'(t_j'-t_{j-1}') \\
&\hspace{1cm}\le
\sum_{j=1}^{r-1}(\theta_j-\theta_j')(t_j-t_{j-1})
+
\sum_{j=1}^{r-1}\theta_j'\bigl(|t_j'-t_j|+|t_{j-1}'-t_{j-1}|\bigr) \\
&\hspace{1cm}\le
b\eta + 2b(r-1)\Delta_b
\le
b(\eta+2r\Delta_b),
\end{align*}
where we used $|t_j'-t_j|\le\Delta_b$ and $|t_{j-1}'-t_{j-1}|\le\Delta_b$.
Finally, since $1-t_{r-1}\ge n^{-\gamma}$ and $t_{r-1}'\le t_{r-1}$, we have $1-t_{r-1}'\ge n^{-\gamma}$, and therefore
\[
\theta_r'
\le \theta_r + b(\eta+2r\Delta_b)n^{\gamma}\le \theta_r + b(\eta+2k\Delta_b)n^{\gamma}.
\]
By the theorem hypothesis, by \eqref{eq:monotonicity_condition_lemma}, $V/k\le 1$, and
\[
\theta_{r-1}-\theta_r
=
\theta_r\bigl(e^{\log\theta_{r-1}-\log\theta_r}-1\bigr)
\ge
a\bigl(e^{V/k}-1\bigr)
\ge
a\frac{V}{k}.
\]
By~\eqref{eq:closeness_of_theta_and_theta'}, $\theta_{r-1}'\ge \theta_{r-1}-b\eta$ and $\theta_r'\le \theta_r+b(\eta+2k\Delta_b)n^{\gamma}$, the condition~\eqref{eq:monotonicity_condition_lemma} guarantees $\theta_r'\le \theta_{r-1}'$, and hence $g\in\mathcal E_{k,n,\beta}$. This proves the feasibility of $g$, and completes the first part.

\medskip
\noindent \textbf{proof of second part:}
Now, for the second part, we start by defining the boundary neighborhoods
\[
I_j := (t_j',\,t_j],\qquad j=1,\ldots,r-1.
\]
Each $I_j$ has length at most $\Delta_b<n^{-\beta}$, and since the sample path belongs in $\mathcal{S}_n(\beta,\gamma)$, defined in~\eqref{eq:good_set}, $I_j$ may contain at most one sample point.
Hence at most $r-1$ indices $t\in\{1,\ldots,n\}$ can change bin membership when the partition changes from $(t_j)_{j=1,\ldots, r}$ to $(t_j')_{j=1,\ldots, r}$.
In particular, we partition indices $i\in\{1,\ldots,n\}$ into two types.
\begin{itemize}
    \item If $x_i\notin\bigcup_{j=1}^{r-1} I_j$, then $x_i$ falls in the same bin under $f$ and $g$, say bin $\ell$, and
    \[
    \bigl|\log f(x_i)-\log g(x_i)\bigr|
    =
    \bigl|\log\theta_\ell-\log\theta_\ell'\bigr|
    \le \eta.
    \]
    \item If $x_i\in\bigcup_{j=1}^{r-1} I_j$, then bin membership may change. There are at most $r-1$ such indices, and for these we use the bound
    \[
    \bigl|\log f(x_i)-\log g(x_i)\bigr|\le \log(b/a)=V.
    \]
\end{itemize}
Combining the two cases yields
\[
\bigl|\mathcal L(f,n)-\mathcal L(g,n)\bigr|
\le n\eta + (r-1)V\le n\eta + (k-1)V.
\]
as claimed.
\hfill $\blacksquare$
\subsection{Proof of results from Section~\ref{sec:application_to_sequential_testing}}
\subsubsection{Proof of Theorem~\ref{thm:asymptotic_log_optimality}: asymptotic log-optimality of calibrators}
We start with recalling that, since $q\in\Dcal$, the log-optimal decreasing calibrator satisfies
\[
h_t^{\mathrm{opt}}=q\quad \text{for all}~~t\in \N.
\]
Consequently, we have
\[
\log M_n(h^{\mathrm{opt}})=\sum_{s=1}^n \log q(p_s),\qquad
\log M_n(\hat h)=\sum_{s=1}^n \log \hat h_s(p_s),
\]
for any online calibrator $\hat{h}=(\hat{h}_s)$. 
% The claim then follows by noting $\textnormal{Risk}(\hat h,n)\ge 0$ and  by Theorem~\ref{thm:expected_excess_logwealth_EWM_and_EAC}, 
% \[
% \textnormal{Risk}(\hat h,n)=\frac{1}{n}
% \mathbb{E}_\bQ\left[\sum_{s=1}^n \log q(p_s)-\sum_{s=1}^n \log \hat h_s(p_s)\right]\;\to\;0
% \qquad\text{as }n\to\infty.
% \]
% This proves the first part.
% \hspace{0.5em}
% \noindent For the second part, we start by writing
Now, we write
\[
\log M_n(\hat h)-\log M_n(h^{\mathrm{opt}})=\sum_{t=1}^n \log\frac{\hat h_t(P_t)}{q(P_t)}=: \sum_{t=1}^n \Delta_t,
\]
and define
\[
a_t:=\E_\bQ[\Delta_t\mid \mathcal F_{t-1}]=\E_\bQ\left[\log\frac{\hat h_t(P_t)}{q(P_t)}\,\middle|\,\mathcal F_{t-1}\right]
=-\mathrm{KL}(q\|\hat h_t).
\]
We decompose
\[
\sum_{t=1}^n \Delta_t = S_n - A_n, \qquad S_n:=\sum_{t=1}^n(\Delta_t-a_t),\qquad A_n:=-\sum_{t=1}^n a_t = \sum_{t=1}^n \mathrm{KL}(q\|\hat h_t).
\]
By construction, $\{S_n\}_{n\ge0}$ is a martingale with respect to
$\{\Fcal_n\}$. Since $q,\hat h_t\in\Dcal_{a,b}$ by the theorem hypothesis, we have
\[
|\Delta_t|\le \log(b/a),\qquad |\Delta_t-a_t|\le 2\log(b/a).
\]
Therefore, by the Azuma--Hoeffding inequality, for any $\varepsilon>0$,
\[
\P_\bQ\bigl(|S_n|\ge n\varepsilon\bigr)\le 2\exp\left(-\frac{n\varepsilon^2}{8(\log(b/a))^2}\right).
\]
Since $\sum_{n=1}^\infty e^{-cn}<\infty$ for any $c>0$, the first
Borel--Cantelli lemma implies
\[
\frac{S_n}{n}\to 0 \qquad \bQ\text{-almost surely}.
\]
It remains to show that $A_n/n\to 0$ almost surely. By
Theorem~\ref{thm:excess_KL_risk_OG_and_EA}, for $\hat h\in\{\hat{f}^{\,\mathrm{OG}}_{a,b},\hat{f}^{\,\mathrm{EA}}_{a,b}\}$, there exists
$C=C(a,b)>0$ such that
\[
\E_\bQ\left[A_n\right]\le C\, n^{1/3}.
\]
Consider the dyadic subsequence $n_k=2^k$. By Markov's inequality, for any $\epsilon>0$,
\[
\P_\bQ\left(\frac{A_{n_k}}{n_k}\ge\epsilon\right)\le\frac{C}{\epsilon}\,2^{-2k/3}.
\]
Since $\sum_{k\ge1}2^{-2k/3}<\infty$, the first Borel--Cantelli lemma yields
\[
\frac{1}{2^k}A_{2^k}
\to 0 \qquad \bQ\text{-almost surely}.
\]
Finally, since $\sum_{t=1}^n \mathrm{KL}(q\|\hat h_t)$ is non-decreasing in $n$,
for $2^k\le n<2^{k+1}$,
\[
\frac{A_n}{n}\le 2\cdot \frac{A_{2^{k+1}}}{2^{k+1}}
\]
which converges to zero almost surely. Combining the two parts proves
\[
\frac{1}{n}\bigl(\log M_n(\hat h)-\log M_n(h^{\mathrm{opt}})\bigr)\to 0
\qquad \bQ\text{-almost surely}.
\]
\hfill $\blacksquare$

\subsubsection{Proof of Corollary~\ref{cor:consistency_calibrators}}
Since $h_t^{\mathrm{opt}}=q\quad \text{for all}~~t\in \N$, by the strong law of large numbers, we have 
\[
\frac{1}{n}\log M_n(h^{\mathrm{opt}})=\frac{1}{n}\sum_{s=1}^n \log q(p_s)\overset{\textnormal{a.s.}}{\longrightarrow} \E_\bQ[\log q]=\int_0^1 q(u) \log q(u) \d u.
\]
By Theorem~\ref{thm:asymptotic_log_optimality}, the proof follows.
\hfill $\blacksquare$
\section{Supporting lemmas}\label{app:supplementary_results}

\begin{lemma}[Histogram form of the constrained Grenander]
\label{lem:bounded_grenander_histogram}
Fix $n\in\mathbb N$ and let $0=x_{(0)}<x_{(1)}<\cdots<x_{(r_n)}<x_{(r_n+1)}=1$ denote the distinct order statistics of the sample path $x_1,\ldots,x_n$, where $r_n\le n$. Further, fix $0<a\le b<\infty$, and let $\tilde{f}^\mle_{n,a,b}$ denote the constrained Grenander estimator, defined in~\eqref{eq:constrained_offline_grenander}.

Then $\tilde{f}^\mle_{n,a,b}$ can be chosen to be a monotone histogram of the form
\[
\tilde{f}^\mle_{n,a,b}(u) = \theta_j, \qquad u\in (x_{(j)},x_{(j+1)}], \quad j=0,\ldots,r_n,
\]
where
\[
b\ge \theta_0\ge \theta_1\ge \cdots \ge \theta_{r_n}\ge a.
\]
In particular, the knots of the MLE may be taken at the order statistics.
\end{lemma}

\begin{proof}
Let $f\in\Dcal_{a,b}$ be any candidate density, and define its
cumulative distribution function
\[
F(x):=\int_0^x f(u)\,du.
\]
Since $f$ is non-increasing, $F$ is concave and absolutely continuous on $[0,1]$, with
\[
a\le F'(x)=f(x)\le b \quad\text{a.e.}
\]
Next, define a piecewise-linear function $\tilde F$ by chord interpolation of $F$ at the knots $\{x_{(j)}\}_{j=0}^{r_n+1}$:
\[
\tilde F(u) := F(x_{(j)})+\theta_j\,(u-x_{(j)}), \qquad u\in[x_{(j)},x_{(j+1)}],
\]
where we define
\[
\theta_j := \frac{F(x_{(j+1)})-F(x_{(j)})}{x_{(j+1)}-x_{(j)}}.
\]
Since $F$ is concave, the slopes must satisfy
\[
\theta_0\ge \theta_1\ge \cdots \ge \theta_{r_n}.
\]
Moreover, since 
\begin{equation}\label{eq:theta_bounds}
\theta_j = \frac{1}{x_{(j+1)}-x_{(j)}}\int_{x_{(j)}}^{x_{(j+1)}} f(u)\,du
\in[a,b],
\end{equation}
$\tilde F$ is a concave CDF whose derivative
$\tilde f:=\tilde F'$ belongs to $\Dcal_{a,b}$. At each distinct observation $x_{(j)}$, by~\eqref{eq:theta_bounds},
\[
\tilde f(x_{(j)})=\theta_{j-1}\ge f(x_{(j)}).
\]
Hence, it follows that
\[
\sum_{i=1}^n \log f(x_i)\;\le\;\sum_{i=1}^n \log \tilde f(x_i).
\]

Applying this argument with
$f=\tilde{f}^\mle_{n,a,b}$ shows that any maximizer of the likelihood
may be replaced by its chord-interpolated version without decreasing the likelihood. Therefore, $\tilde{f}^\mle_{n,a,b}$ can be chosen to be piecewise
constant with knots at the order statistics, as claimed.
\end{proof}

\begin{lemma}\label{lem:min_distance_in_dimension_d}
Let $X_1,\ldots,X_n\iidsim\bQ$, where $\bQ$ is supported on $[0,1]$ and has density $q\in\Dcal_{a,b}$. Fix $\beta>2$ and $\gamma>1$.
Then, there exists $N\equiv N(\delta,\beta,\gamma)$ such that for all $n>N$,
\[
\P_\bQ\!\left(\min_{1\le i<j\le n}|X_i-X_j| \ge n^{-\beta}, \max_{1\le i\le n} X_i \le 1-n^{-\gamma}\right)\;\ge\; 1-\delta.
\]
\end{lemma}
\begin{proof}
By a union bound,
\[
\P_\bQ\!\left(\min_{1\le i<j\le n}|X_i-X_j|\le \tau\right)
\le
\binom{n}{2}\,\P_{X'_1,X'_2\iidsim \bQ}(|X'_1-X'_2|\le \tau).
\]
Since $\bQ$ has density bounded above by $b$,
\[
\P_{X'_1,X'_2\iidsim \bQ}(|X'_1-X'_2|\le \tau\mid X_1)\le \frac{2bn^2}{2n^\beta}=\frac{b}{n^{\beta-2}}.
\]
Next, we note that for any distribution $\bQ$ with density $q\in \Dcal_{a,b}$, $Q$ is stochastically dominated by $\mathrm{Unif}[0,1]$. Therefore, 
\[
\P_\bQ\!\left(\max_{1\le i\le n} X_i \ge 1-n^{-\gamma}\right) \le \P_{\mathrm{Unif}[0,1]}\!\left(\max_{1\le i\le n} X_i \ge 1-n^{-\gamma}\right)
\]
Now, by a union bound, 
\[
\P_{\mathrm{Unif}[0,1]}\!\left(\max_{1\le i\le n} X_i \ge 1-n^{-\gamma}\right)\le n^{-\gamma+1}
\]
Combining both probability bounds by a union bound, the existence of a $N\equiv N(\delta,\beta,\gamma)$ is immediate, such that for all $n>N$,
\[
\P_\bQ\!\left(\min_{1\le i<j\le n}|X_i-X_j| \ge n^{-\beta}, \max_{1\le i\le n} X_i \le 1-n^{-\gamma}\right)\;\ge\; 1-\delta.
\]
This completes the proof.
\end{proof}
\begin{lemma}\label{lem:L2_H_KL_equivalence}
Let $f,g\in \mathcal D_{\gamma_0,\gamma_1}$, where $0<\gamma_0\le \gamma_1<\infty$.
Then the following bounds hold:
\begin{enumerate}
    \item $\mathrm{KL}(f\|g)\;\le\;\frac{1}{\gamma_0}\int_0^1 (f(u)-g(u))^2\,du$.
    \item $\mathrm{KL}(f\|g)\;\le\;\frac{8\gamma_1}{\gamma_0}\,\hel^2(f,g)$
    \item $\|\cdot\|_2$ and $\hel$ are equivalent on $\mathcal D_{\gamma_0,\gamma_1}$ in the sense that
    \[
    \gamma_0 \int_0^1 \bigl(\sqrt{f(u)}-\sqrt{g(u)}\bigr)^2\,du
    \;\le\;
    \int_0^1 \bigl(f(u)-g(u)\bigr)^2\,du
    \;\le\;
    4\gamma_1 \int_0^1 \bigl(\sqrt{f(u)}-\sqrt{g(u)}\bigr)^2\,du.
    \]
\end{enumerate}
\end{lemma}

\begin{proof}
We prove each part separately.\\[0.5em]
\noindent\textbf{Proof of part 1.}\hspace{1em}
We write
\[
\mathrm{KL}(f\|g)=\int_0^1 f(x)\log\frac{f(x)}{g(x)}\,dx
=\int_0^1 f(x)\log\Bigl(1+\frac{f(x)-g(x)}{g(x)}\Bigr)\,dx.
\]
Since $f(x)\ge \gamma_0$, we have $\frac{f(x)-g(x)}{g(x)}>-1$, and hence $\log(1+u)\le u$ yields
\[
\mathrm{KL}(f\|g)\le\int_0^1 f(x)\frac{f(x)-g(x)}{g(x)}\,dx=\int_0^1 \frac{(f(x)-g(x))^2}{g(x)}\,dx+\int_0^1 (f(x)-g(x))\,dx.
\]
The last integral is zero because $f$ and $g$ are in $\mathcal{D}$, so
\[
\mathrm{KL}(f\|g)\le\int_0^1 \frac{(f(x)-g(x))^2}{g(x)}\,dx\le\frac{1}{\gamma_0}\int_0^1 (f(x)-g(x))^2\,dx,
\]
which proves the first bound.

\vspace{0.5em}
\noindent\textbf{Proof of part 2.}\hspace{1em}For the second bound, note that
\[
(f(x)-g(x))^2=(\sqrt{f(x)}-\sqrt{g(x)})^2(\sqrt{f(x)}+\sqrt{g(x)})^2.
\]
Since $f(x),g(x)\le \gamma_1$, we have $(\sqrt{f(x)}+\sqrt{g(x)})^2\le 4\gamma_1$, and thus
\[
\int_0^1 (f(x)-g(x))^2\,dx\le 4\gamma_1 \int_0^1 (\sqrt{f(x)}-\sqrt{g(x)})^2\,dx.
\]
Combining this with the first bound and $\hel^2(f,g)=\frac{1}{2}\int_0^1(\sqrt{f(x)}-\sqrt{g(x)})^2\,dx$
gives
\[
\mathrm{KL}(f\|g)\le\frac{1}{\gamma_0}\int_0^1 (f(x)-g(x))^2\,dx
\le\frac{4\gamma_1}{\gamma_0}\int_0^1(\sqrt{f(x)}-\sqrt{g(x)})^2\,dx=\frac{8\gamma_1}{\gamma_0}\hel^2(f,g).
\]

\vspace{0.5em}
\noindent\textbf{Proof of part 3.}\hspace{1em}Finally, using again $(f(x)-g(x))^2=(\sqrt{f(x)}-\sqrt{g(x)})^2(\sqrt{f(x)}+\sqrt{g(x)})^2$ and the bounds
\[
(\sqrt{f(x)}+\sqrt{g(x)})^2 \le 4\gamma_1,
\qquad
(\sqrt{f(x)}+\sqrt{g(x)})^2 \ge (\sqrt{f(x)})^2 \ge \gamma_0,
\]
we obtain
\[
\gamma_0\int_0^1(\sqrt{f(x)}-\sqrt{g(x)})^2\,dx
\le\int_0^1(f(x)-g(x))^2\,dx\le 4\gamma_1\int_0^1(\sqrt{f(x)}-\sqrt{g(x)})^2\,dx,
\]
which implies the displayed equivalence.
\end{proof}

\begin{lemma}\label{lem:mixability_of_logloss}
Let $g_1,\ldots,g_M\in \Dcal$ , and let $\pi_1,\ldots,\pi_M$ be prior weights
with $\pi_j>0$ for all $j$ and $\sum_{j=1}^M \pi_j=1$.
Given observations $x_1,\ldots,x_n$, define for each $t\in[n]$ the weighted
mixture
\[
\hat f_t(x)=\sum_{j=1}^M w_{t-1}(j)\,g_j(x),
\qquad
w_{t-1}(j)
=\frac{\pi_j \prod_{s=1}^{t-1} g_j(x_s)}
{\sum_{\ell=1}^M \pi_\ell \prod_{s=1}^{t-1} g_\ell(x_s)}.
\]
Then, for every $j\in\{1,\ldots,M\}$,
\[
\sum_{t=1}^n \log \hat f_t(x_t)\;\ge\;
\sum_{t=1}^n \log g_j(x_t)\;+\;\log\pi_j.
\]
In particular, if $\pi_j=1/M$ for all $j\in[M]$, then
\[
\sum_{t=1}^n \log \hat f_t(x_t)\;\ge\;\max_{1\le j\le M}\sum_{t=1}^n \log g_j(x_t)\;-\;\log M.
\]
\end{lemma}

\begin{proof}
The result is standard; see, for instance, \citet[Section 9.2]{cesa2006prediction}. We include the argument here for completeness.

Define the normalizing constants
\[
Z_t:=\sum_{\ell=1}^M \pi_\ell \prod_{s=1}^{t} g_\ell(x_s),
\qquad
Z_0:=\sum_{\ell=1}^M \pi_\ell = 1.
\]
By the definition of $w_{t-1}$, for each $t\in[n]$,
\[
\hat f_t(x_t)=\sum_{j=1}^M w_{t-1}(j)\,g_j(x_t)
=\frac{\sum_{j=1}^M \pi_j\prod_{s=1}^{t-1}g_j(x_s)\,g_j(x_t)}{\sum_{\ell=1}^M \pi_\ell\prod_{s=1}^{t-1}g_\ell(x_s)}=\frac{Z_t}{Z_{t-1}}.
\]
Therefore,
\[
\sum_{t=1}^n \log \hat f_t(x_t)=\sum_{t=1}^n \bigl(\log Z_{t}-\log Z_{t-1}\bigr)=\log Z_n-\log Z_0
=\log Z_n.
\]
Since $Z_n=\sum_{\ell=1}^M \pi_\ell\prod_{s=1}^n g_\ell(x_s)\ge \pi_j\prod_{s=1}^n g_j(x_s)$,
we obtain
\[
\log Z_n\ge\log \pi_j +\sum_{t=1}^n \log g_j(x_t),
\]
which proves the first claim. The second follows by taking $\pi_j=1/M$ and minimizing over $j$.
\end{proof}

\end{document}